\pdfoutput=1

\documentclass[11pt]{article}

\usepackage{subcaption}

\usepackage[final]{acl}

\usepackage{times}
\usepackage{latexsym}
\usepackage{amsthm}
\newtheorem{theorem}{Theorem}[section]
\newtheorem{assumption}{Assumption}[section]
\usepackage[T1]{fontenc}

\usepackage[utf8]{inputenc}

\usepackage{microtype}

\usepackage{inconsolata}

\usepackage{graphicx}

\usepackage{amsmath}
\usepackage{amssymb}
\DeclareMathOperator*{\argmin}{argmin}

\usepackage{algorithm}
\usepackage{color,soul}
\usepackage{algpseudocode}

\usepackage{csvsimple}

\title{Dynamic Learning Rate Scheduling based on Loss Changes Leads to Faster Convergence}

\author{Shreyas Subramanian \\
  Amazon Web Services \\
  Seattle, Washington \\
  \texttt{subshrey@amazon.com} \\\And
  Bala Krishnamoorthy \\
  Amazon Web Services \\
  Seattle, Washington \\
  \texttt{bkrism@amazon.com} \\\And
  Pranav Murthy \\
  Amazon Web Services \\
  Seattle, Washington \\
  \texttt{pranavvm@amazon.com} \\
  }

\begin{document}
\maketitle
\begin{abstract}
Despite significant advances in optimizers for training, most research works use common scheduler choices like Cosine or exponential decay. In this paper, we study \emph{GreedyLR}, a novel scheduler that adaptively adjusts the learning rate during training based on the current loss. To validate the effectiveness of our proposed scheduler, we conduct experiments on several NLP, CV, and LLM tasks with up to $7B$ parameters, including both fine-tuning and pre-training experiments. The results show that our approach outperforms several state-of-the-art schedulers in terms of accuracy, speed, and convergence. We also provide a theoretical analysis of the GreedyLR algorithm, including a proof of convergence and derivation of the optimal scaling factor $F$ that maximizes the convergence rate, along with experiments to show robustness of the algorithm to realistic noisy landscapes. Our scheduler is easy to implement, computationally efficient, and could be considered a good default scheduler for training.
\end{abstract}

\section{Introduction}
Selecting a learning rate (LR) scheduler for training is important, but is often done with minimal thought. Many recent works default to using specific LR schedulers such as the Cosine Annealing scheduler, frequently without a strong technical justification for their choice. 

As a first form of changing learning rates adaptively through training, several adaptive optimization methods have been proposed, such as Adam (Adaptive Moment Estimation) \cite{Kingma2014AdamAM} and RMSProp (Root Mean Square Propagation), which dynamically adjust the learning rate based on gradients and the history of updates. However, these adaptive optimizers often underperform in practice with their default settings \cite{Wilson2017TheMV,Macdo2021TrainingAS}. Techniques proposed by \cite{Vaswani2019PainlessSG,Armijo1966MinimizationOF} aim to determine the optimal LR at each training step by treating it as a line search problem. These methods still use a fixed, predetermined schedule.

The main drawback of fixed schedules is their generality, which prevents adaptation to the specific characteristics of the optimization problem or the model architecture. Different problems and architectures often require distinct LR schedules for optimal performance. Therefore, there is a pressing need for a learning rate scheduler that is both simple and adaptable to the specific optimization problem.

There is a growing trend towards using learning rate schedules that adjust the LR during training. In our work, we propose a novel and simple scheduler called \textit{GreedyLR}, which adaptively chooses the learning rate. Our contributions are as follows: \begin{enumerate}
    \item We conduct a variety of experiments from small models to Large Language Models (LLMs) with billions of parameters to validate performance of the scheduler across model scalses, use cases and datasets\\
    \item We demonstrate GreedyLR's effectiveness across both fine-tuning and pre-training paradigms, establishing its utility as a general-purpose scheduler for diverse training scenarios\\
    \item We study critical hyperparameters, as well as the robustness of the scheduler to simulated noisy environments to encourage using \textit{GreedyLR} as a default scheduler choice in training experiments. \\
\end{enumerate}

\section{Related Work}

The scheduling of learning rates is a critical factor in the training of deep neural networks (DNNs), influencing both convergence speed and final model performance. \cite{Macdo2021TrainingAS,Dauphin2014IdentifyingAA} suggest that neural network training occurs in phases, advocating for different learning rates at each phase to facilitate convergence. \cite{Smith2017SuperconvergenceVF,Smith2015CyclicalLR} employ cyclical variations of the learning rate based on preset heuristics to improve training dynamics. \cite{nakamura2021learning} propose a novel annealing schedule combining a sigmoid function with a warmup phase that maintains large learning rates during early and middle training stages while smoothing transitions to avoid abrupt changes in step size. \cite{yedida2019novel} derive a theoretical framework for dynamically computing learning rates based on the Lipschitz constant of the loss function, though their experiments indicate challenges in generalizing across architectures like ResNets. \cite{9534014} introduce an Adaptive Scheduler for Learning Rate (ASLR) that requires minimal hyperparameter tuning and adapts based on validation error trends, reducing computational burden while remaining effective across various network topologies. \cite{kim2021automated} propose an automated scheduler combining adaptive warmup and predefined decay phases for large-batch training, achieving superior performance with stochastic optimizers like AdamP and LAMB. \cite{defazio2023and} present a refined adaptive scheduling approach that focuses on the last iterate rather than the average, adjusting schedules based on observed gradient norms and often outperforming popular schedules like cosine annealing. \cite{app112110184} propose Adacomp, a zeroth-order method adjusting learning rates based on loss values alone that shows robustness across datasets and architectures, though it falls short of state-of-the-art adaptive methods in achieving maximum validation accuracy. \cite{jin2021autolrs} leverage Bayesian optimization in AutoLRS to dynamically search for optimal learning rates during training, balancing exploration and exploitation to yield significant speedups over state-of-the-art schedules. \cite{10.1145/3377930.3390158} explore evolutionary approaches in AutoLR, which evolves learning rate policies specific to neural network architectures using Structured Grammatical Evolution to generate efficient schedules. \cite{yedida2021lipschitzlr} provide a theoretical framework for adaptive learning rates based on the Lipschitz constant that achieves faster convergence by analytically determining optimal rates for various optimizers.

Collectively, these studies highlight the diversity and complexity of adaptive learning rate scheduling methods, each with unique strengths and suitable applications, contributing significantly to the efficient training of deep learning models. Despite these advancements, limitations remain. Many methods require substantial computational resources, making them less accessible for practitioners with limited resources \cite{jin2021autolrs}. Methods based on theoretical frameworks like Lipschitz constants may face challenges in accurately estimating necessary parameters in practical, noisy environments, leading to suboptimal performance \cite{yedida2019novel, yedida2021lipschitzlr}. Additionally, the complexity of certain algorithms, such as ASLR and those utilizing advanced statistical models, can make them difficult to implement and tune without deep expertise, thus limiting their usability \cite{khodamoradi2021aslr, kim2021automated}. Moreover, despite claims of generalizability, many techniques show varying degrees of effectiveness across different architectures and datasets, indicating that no single method universally outperforms others \cite{nakamura2021learning, defazio2023and}. 

Finally, the lack of standardization in benchmarking and evaluation methodologies for schedulers makes it challenging to directly compare the effectiveness of different scheduling approaches, further complicating the selection of the most appropriate method for a given application. We understand the ``no free lunch'' principle, and the fact that coming up with a scheduler that outperforms all other schedulers in all use cases may not be possible even with our contributions below, but we believe we can come up with a good, sensible default choice that is simple to implement and reliable in terms of performance. Next, we describe the GreedyLR scheduler, followed by theoretical proofs of convergence and experiments with LLMs.

\section{\emph{GreedyLR} Scheduler}

The \textit{GreedyLR} scheduler adjusts the learning rate based on changes in loss. Algorithm \ref{alg:greedylr} in the Appendix provides a detailed view of the implementation. In its simplest form, the scheduler uses a fixed factor $F$ between $(0,1)$ to modify the LR: it multiplies the rate by this factor if the loss worsens to decrease LR, or divides by the same factor $F$ if the loss improves to increase LR. The intuition behind using the scaling factor $F$ to increase or decrease the learning rate based on the change in loss. If the loss value decreases ($l_t < l_{t-1}$) over time, it suggests that we are moving in a direction that potentially reduces the objective function. In this case, we want to take a larger step in the same direction by increasing the learning rate ($\gamma_t = \gamma_{t-1} / F$, where $F < 1$). We do the opposite if the loss value increases ($l_t \ge l_{t-1}$). Although preliminary in nature, we refer the interested reader to the appendix that discusses theoretical properties of such an algorithm when applied to SGD. We show that:

\begin{enumerate}
    \item \textit{GreedyLR} converges with a rate of $O(1/T)$ for the expected sub-optimality of the average iterate $\bar{x}_T$. [Theorem \ref{thm:main}]. We support this with various real world fine tuning and pre-training experiemnts, along with experiments that show robustness to noise.
    \item The optimal value of the factor $F$ is $F = 1 - \frac{1}{L_{\max}}$, where $L_{\max}$ is the smoothness constant of the objective function.  [Theorem \ref{thm:optimal_scaling_factor}]. We support this with experiment results from a $F$-sweep in Section~\ref{sec:Fsweep}
\end{enumerate}

\section{Experimental Results}

We evaluated GreedyLR across diverse model scales and tasks to assess its
effectiveness as a general-purpose scheduler. Our experiments span models from
tens of millions to 7 billion parameters, covering NLP, CV, and LLM tasks in
both fine-tuning and pre-training paradigms. Figure~\ref{fig:perf-vs-model-size}
summarizes performance across parameter scales, showing that GreedyLR matches or
exceeds baseline schedulers in the majority of experiments, with particularly
strong benefits in the 1-200M parameter range (see positive final loss delta except for a few outliers).

Key findings include: (1) For small models (<500M parameters), GreedyLR performs
as good or better than popular schedulers in 86.73\% of experiments across 132
training runs, with average loss improvement of 0.16 and maximum benefit of 2.3.
(2) For large models (500M-7B parameters), GreedyLR achieves 83.33\%
as-good-or-better performance in fine-tuning, with strong gains (up to 47\%)
during early training. (3) In pre-training on Llama-3.2-1B using
RedPajama-arxiv, GreedyLR achieves 5.4\% lower final loss versus Cosine
scheduling. (4) Empirical analysis reveals a stability threshold at $F \geq 0.5$
for the scaling factor, above which performance is robust (within 1.5\%
variation), eliminating precise hyperparameter tuning. In the next few sub-sections, we will dive deeper into the above results. 

\begin{figure}[hbt!]
    \centering
    \includegraphics[width=1.1\linewidth]{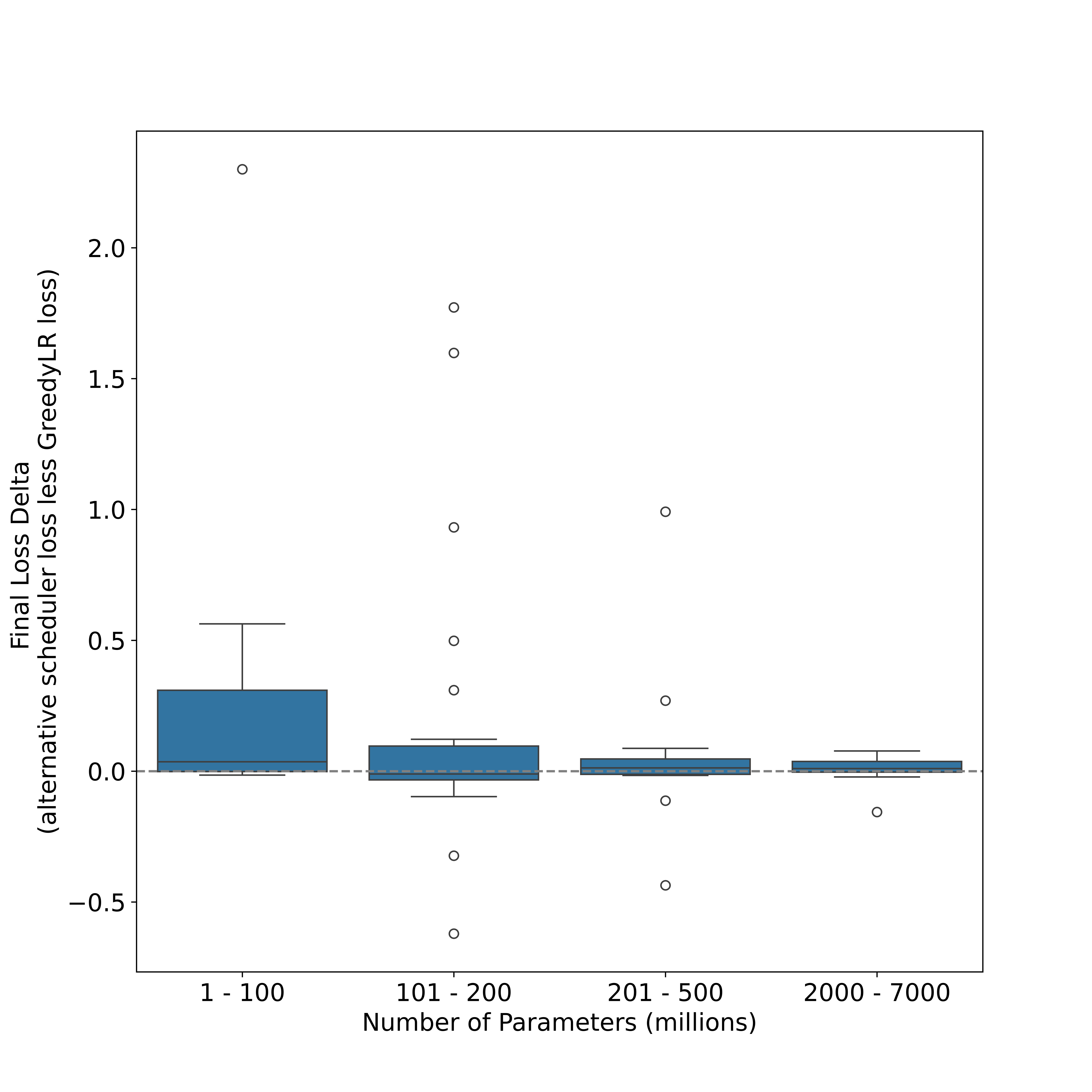}
    \caption{Training Performance (final loss delta) vs Model Size (number of parameters)}
    \label{fig:perf-vs-model-size}
\end{figure}

\subsection{Small Model Results}

``Small Model'' hererefers to models with $\leq$500M parameters. We conducted 132
experiments across 16 model architectures (including Pegasus, BERT, T5, BART,
ResNet, ViT, Camembert families) and 15 diverse datasets spanning translation
(WMT16, Opus100), QnA (SQUAD, Adversarial QA, Quoref), summarization (XSUM,
Amazon reviews), NER (Conllpp, Wikiann, Xglue), and image tasks (CIFAR-10/100,
Tiny ImageNet, sidewalk-semantic). We tested 4 optimizers (AdamW, Adafactor,
Adagrad, SGD) with 5 schedulers (Linear, Cosine, Polynomial,
Constant+warmup, plus GreedyLR).

All experiments ran on ml.g4dn.16xlarge Amazon SageMaker instances with
identical seeds, initial learning rates, and Huggingface defaults. For
GreedyLR: patience=10, min\_lr=$10\%$ of initial LR, smoothing window=50. We
measured loss at $10\%$, $50\%$, and $100\%$ of training steps (typically
1000-5000 steps). See Appendix Table~\ref{tab:doe} for complete experimental
design.

Tables~\ref{tab:0a} and~\ref{tab:0b} summarize the detailed task results. Table~\ref{tab:0a}
categorizes instances where GreedyLR, using the same optimizer as baseline,
clearly outperforms (``yes''), outperforms with no significant difference
(``yes*''), clearly underperforms (``no''), or underperforms with no
significant difference (``no*''). Insignificant difference at any stage is
defined as absolute loss difference < ±0.1, labeled ``yes*'' and ``no*''.

Table~\ref{tab:0b} presents summary statistics from our small model
experiments. Across use cases, GreedyLR is as good or better (``yes'',
``yes*'', ``no*'') than tested schedulers $>86\%$ of the time, and better
(``yes'', ``yes*'') $57\%$ of the time. It clearly outperforms (``yes'') in
$~25\%$ of instances. For final loss, GreedyLR with same base optimizer
outperforms $>91\%$ of runs.

\begin{table}[htbp]
\centering
\caption{Results summary - counts of how often GreedyLR beats other schedulers
across all small model experiments at three stages (10, 50, 100\% of max
steps). ``Yes'' = GreedyLR clearly outperforms, ``no'' = clearly underperforms;
* indicates loss delta < $\pm$0.1 (no significant difference).}
    \begin{tabular}{|l|l|l|l|l|p{2cm}|}
        \hline
        yes & yes* & no & no* & sum & Final loss delta in +/- 0.1 \\
        \hline
        48  & 64   & 26 & 58  & 196 & 39                          \\
        \hline
    \end{tabular}
    \label{tab:0a}
\end{table}
\begin{table}[htbp]
\caption{Summary of performance calculated from Table~\ref{tab:0a}}
    \centering
    \begin{tabular}{|p{4cm}|p{3cm}|}
        \hline
        \textbf{Metric}           & \textbf{Percentage (\%)}          \\ \hline
        Overall as good or better & 86.73                             \\
        Overall Better            & 57.14                             \\
        Overall Worse             & 13.27                             \\
        Overall as good           & 62.24                             \\

        Clearly better            & 24.49                             \\
        \hline
        \textbf{Metric}           & \textbf{Baseline - GreedyLR loss} \\ \hline
        Average benefit           & 0.16                              \\
        Max benefit               & 2.3                               \\
        Max deficit               & -0.62                             \\
        \hline
    \end{tabular}
    
    \label{tab:0b}
\end{table}
\begin{table}[htbp]
\caption{Summary of performance by stage showing percentage of times GreedyLR
is as-good-or-better}
    \centering
    \begin{tabular}{|l|l|l|}
        \hline
        Stage 1 (10\%) & Stage 2 (50\%) & Stage 3 (100\%) \\
        \hline
        92.42          & 81.81          & 85.94           \\
        \hline
    \end{tabular}
    \label{tab:0c}
\end{table}

\subsection{Large Model Results}

``Large Model'' in this context refers to models with parameters greater than 500 million parameters up to 7B parameters. We conducted experiments across multiple model architectures to evaluate GreedyLR's effectiveness in both fine-tuning and pre-training scenarios.

\subsubsection{Fine-Tuning Experiments}

We conducted 8 fine-tuning experiments across three popular model architectures with varying parameter sizes: Microsoft's Phi-2 (2 billion parameters), TII UAE's Falcon 7B (7 billion parameters), and Google's Gemma 7B (7 billion parameters). These large model architectures were fine-tuned using three different modalities of datasets, as summarized in Table~\ref{tab:doe_large_models}. 

\begin{table}[hbt!]
\caption{Design of Experiments (DoE) for Large Model Architectures}
\begin{tabular}{|p{3cm}|p{4cm}|}
\hline
\textbf{Models} & \textbf{Datasets} \\
\hline
Microsoft Phi2 2B & w601sxs/simpleCoT \\
TII UAE Falcon 7B  & b-mc2/sql-create-context \\
Google Gemma 7B  & jpacifico/French-Alpaca-dataset-Instruct-55K \\
\hline
\textbf{Optimizers} & \textbf{Schedulers} \\
\hline
AdamW & GreedyLR (ours) \\
 & Cosine \\
\hline
\end{tabular}
\label{tab:doe_large_models}
\end{table}

A brief description of the datasets from Huggingface used for fine-tuning follows:

\begin{enumerate}
    \item \emph{w601sxs/simpleCoT}: An instruct-tune format dataset designed to adapt pretrained models to the instruct format. We constructed simpleCoT from several Open source datasets on Huggingface with open licenses including Orca, Wizard LM, Kaist, and AlpacaCoT.\cite{orca, alpaca, kaist, wizardlm}
    \item \emph{b-mc2/sql-create-context}: A collection of natural language queries in the instruct-tune format which is a combination of Seq2SQL and Spider datasets.\cite{spider1, spider2}
    \item \emph{jpacifico/French-Alpaca-dataset-Instruct-55K}: A synthetically generated collection of 55K French language Alpaca formatted instructions.
\end{enumerate}

Overall results for fine tuning of larger models  comparing GreedyLR performance with schedulers tested are shown in Tables~\ref{tab:1b} and \ref{tab:1c}, which are derived from Table~\ref{tab:1a}. Note that for comparison, GreedyLR is compared to the Cosine scheduler, a default implementation in the Huggingface Python library. For LLM experiments, we see that across three stages of fine-tuning, GreedyLR is 83.33\% as good or better than the baseline; GreedyLR is clearly better 62.5\% of the measured datapoints. While no scheduler can show superior performance across all stages, datasets, and model baselines, we point out that in the experiments run for LLMs, GreedyLR has a net positive benefit, with a maximum benefit of 47\% and maximum deficit of 28\%. Specifically, we see good improvement in each of the three stages (10\%, 50\%, and 100\% of the max steps), with an uplift in the early stages of convergence.

\begin{table}[h]
\centering
    \begin{tabular}{|l|l|l|l|l|p{2cm}|}
        \hline
        yes & yes* & no & no* & sum & Final loss delta in +/- 1\% \\
        \hline
        15  & 4   & 4 & 1  & 24 & 6                          \\
        \hline
    \end{tabular}
 
    \caption{Results summary - counts of how often GreedyLR scheduler beats other schedulers across LLM experiments, and at three stages (10, 50 and 100\% of max steps). Yes means that GreedyLR with the same base optimizer beats the scheduler in comparison, and no means that it does not. * indicates that the loss delta at the measured point is less than $+/- 1\%$}
       \label{tab:1a}
\end{table}
\begin{table}[h]
    \centering
    \begin{tabular}{|p{4cm}|p{3cm}|}
        \hline
        \textbf{Metric}           & \textbf{Percentage (\%)}          \\ \hline
        Overall as good or better & 83.33                            \\
        Overall Better            & 79.16                            \\
        Overall Worse             & 16.66                            \\
        Overall as good           & 20.83                             \\

        Clearly better            & 62.5                             \\
        \hline
        \textbf{Metric}           & \textbf{Baseline - GreedyLR loss} \\ \hline
        Average benefit           & 2.89\%                              \\
        Max benefit               & 47\%                               \\
        Max deficit               & 28\%                            \\
        \hline
    \end{tabular}
    
    \caption{Summary of performance calculated from Table \ref{tab:1a}}
    \label{tab:1b}
\end{table}
\begin{table}[h]
    \centering
    \begin{tabular}{|l|l|l|}
        \hline
        Stage 1 (10\%) & Stage 2 (50\%) & Stage 3 (100\%) \\
        \hline
        87.5          & 75          & 75          \\
        \hline
    \end{tabular}
    
    \caption{Summary of performance calculated by stage showing what percentage of times GreedyLR is overall as good, or better }
    \label{tab:1c}
\end{table}

Figure~\ref{fig:3}a shows the dynamic learning rate (LR) generated by GreedyLR, in comparison to the standard Cosine scheduler. The loss curve for Gemma-7B fine-tuning (Figure~\ref{fig:3}b) shows accelerated early convergence for GreedyLR compared to Cosine. 
\begin{figure}[h!]
    \centering
    \begin{subfigure}[b]{\columnwidth}
        \includegraphics[width=\linewidth]{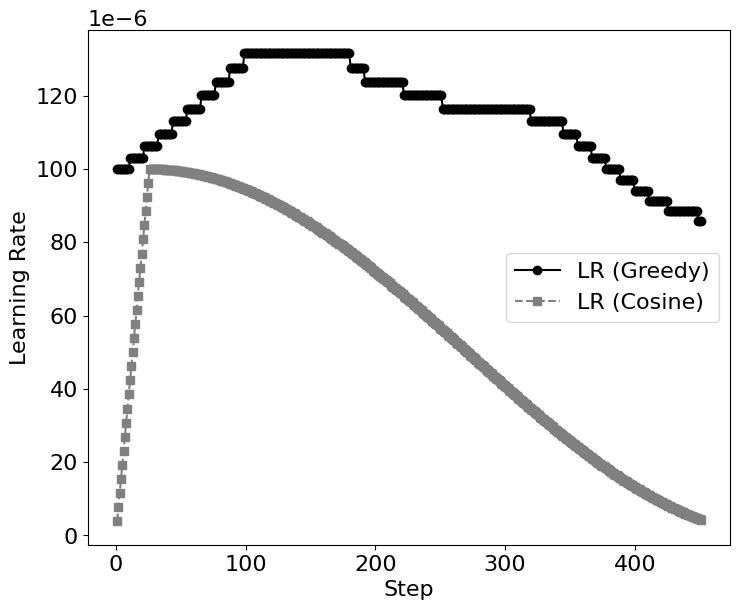}
        \caption{Learning Rate}
        \label{fig:3:lr}
    \end{subfigure}
    \hfill
    \begin{subfigure}[b]{\columnwidth}
        \includegraphics[width=\linewidth]{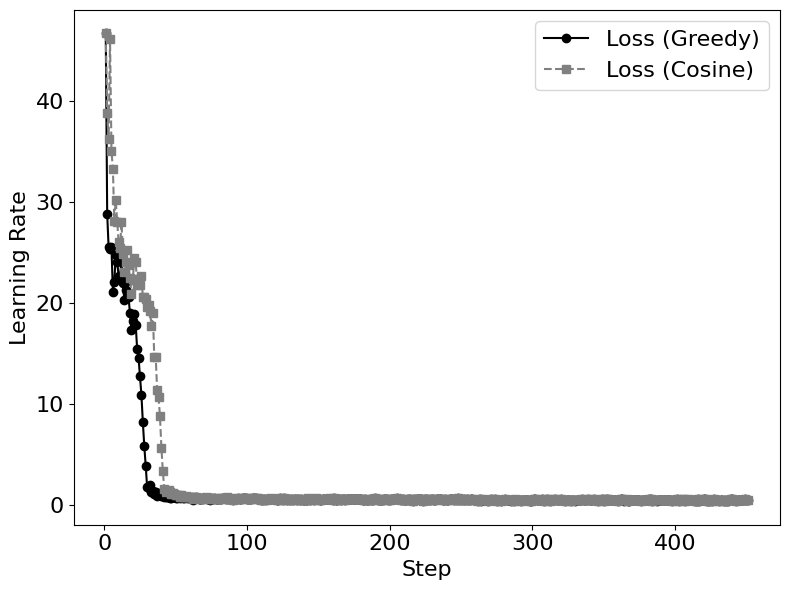}
        \caption{Loss}
        \label{fig:3:loss}
    \end{subfigure}
    \caption{Google Gemma-7b, showing (a) learning rate schedules and (b) loss trajectories for the Greedy and Cosine schedulers. We observe that the Greedy scheduler significantly outperforms the Cosine scheduler during the early stages of training.}
    \label{fig:3}
\end{figure}

GreedyLR significantly outperforms Cosine during early
fine-tuning stages (when larger gradient updates and domain adaptation occur)
and performs as-well-or-better in later stages. Detailed results in Appendix
Table~\ref{tab:results_summary} show GreedyLR outperforms Cosine in all
experiments during the first 10\% of training and in 5 of 6 experiments
throughout training.

\subsubsection{Pre-Training Experiments}

To assess effectiveness beyond fine-tuning, we pre-trained Meta's Llama-3.2-1B
on RedPajama-arxiv for 1000 steps ($\gamma_0 = 2 \times 10^{-4}$, warmup=100,
batch size=1 with gradient accumulation=32, bf16). GreedyLR used $F=0.95$,
min\_lr=$1.85 \times 10^{-5}$, smoothing enabled.

GreedyLR achieves 1.0\%, 3.0\%, and 5.4\% lower loss at 10\%, 50\%, and 100\%
of training respectively (final: 2.16 vs 2.28). Unlike fine-tuning where
early-stage benefits dominate, pre-training shows accelerating advantages,
suggesting GreedyLR's loss-based adaptation is particularly effective in
high-variance settings without prior task knowledge. See Appendix
Figure~\ref{fig:llama_pretrain_appendix} for detailed learning rate schedules
and loss trajectories.

\subsection{Stability Threshold for Scaling Factor $F$}
\label{sec:Fsweep}
While Theorem~\ref{thm:optimal_scaling_factor} establishes the theoretical
optimal value $F = 1 - \frac{1}{L_{\max}}$, $L_{\max}$ is typically unknown in
practice. We conducted a systematic F-sweep using Microsoft Phi-2 (2B
parameters) on w601sxs/simpleCoT with $F \in \{0.25, 0.50, 0.75, 0.99\}$ over
250 steps.

Figure~\ref{fig:f_sweep} reveals a critical stability threshold: $F=0.25$
caused catastrophic divergence (final loss 7.78 vs initial 2.28), while all
$F \geq 0.5$ achieved stable convergence with nearly identical performance
(losses 1.89, 1.92, 1.91---within 1.5\%). This demonstrates that practitioners
need only ensure $F \geq 0.5$ for stability, eliminating precise hyperparameter
tuning. See Appendix Figure~\ref{fig:f_sweep_detailed} for detailed analysis
including learning rate dynamics and zoomed comparisons.

\begin{figure}[h!]
    \centering
    \includegraphics[width=\linewidth]{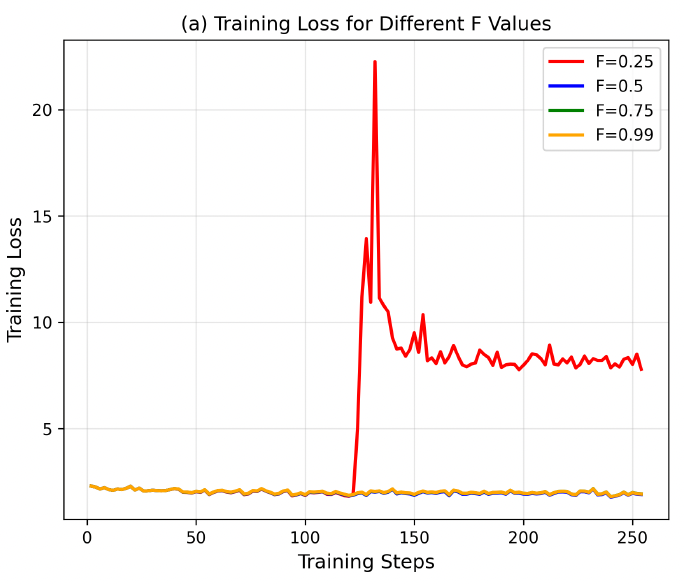}
    \caption{Training loss trajectories for different scaling factor $F$ values
on Microsoft Phi-2 fine-tuning, demonstrating the critical stability threshold
at $F \geq 0.5$. $F=0.25$ causes catastrophic divergence while all $F \geq 0.5$
achieve stable convergence with similar performance (within 1.5\%). See Appendix
Figure~\ref{fig:f_sweep_detailed} for detailed analysis.}
    \label{fig:f_sweep}
\end{figure}

\subsection{Robustness Experiments}
\label{sec:robustness_main}

We conducted $8100$ training experiments to evaluate GreedyLR's robustness against real-world training perturbations. Our experimental design (detailed in Appendix~\ref{sec:robustess}) includes five noise types applied as additive perturbations to the loss function: Gaussian noise (stochastic gradient errors), periodic spike noise (scheduled disruptions every 50-100 steps), random spike noise (2\% probability, simulating hardware glitches), adversarial noise (opposing optimization progress), and a clean baseline. We evaluated four schedulers across 12 neural architectures, with GreedyLR receiving comprehensive evaluation ($n=3241$ runs) compared to baseline schedulers ($n \approx 1620$ each).

Figure~\ref{fig:figrob1} demonstrates GreedyLR's superior performance, achieving the lowest median final loss (0.148) compared to cosine annealing (0.232), cosine with restarts (0.226), and exponential decay (0.249). The performance heatmap (Figure~\ref{fig:figrob3}) reveals GreedyLR's consistent robustness across all noise conditions, with particularly strong performance under adversarial, Gaussian, and spike perturbations where traditional schedulers show high variability.

\begin{figure}[h]
    \centering
    \includegraphics[width=\columnwidth]{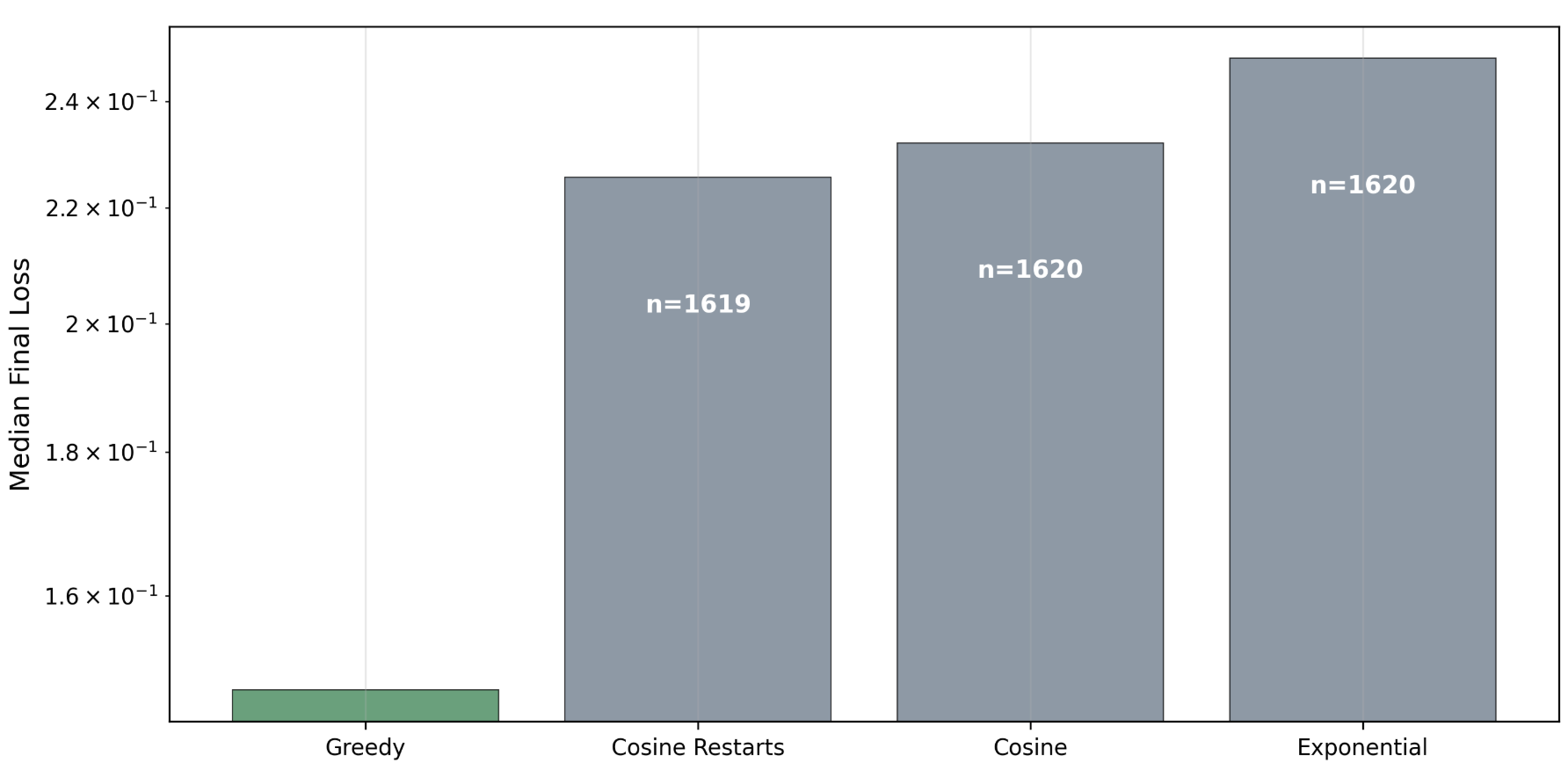}
    \caption{Median final loss comparison across all experiments. GreedyLR achieves 37\% lower median loss than the best traditional scheduler.}
    \label{fig:figrob1}
\end{figure}

\begin{figure}[h]
    \centering
    \includegraphics[width=\columnwidth]{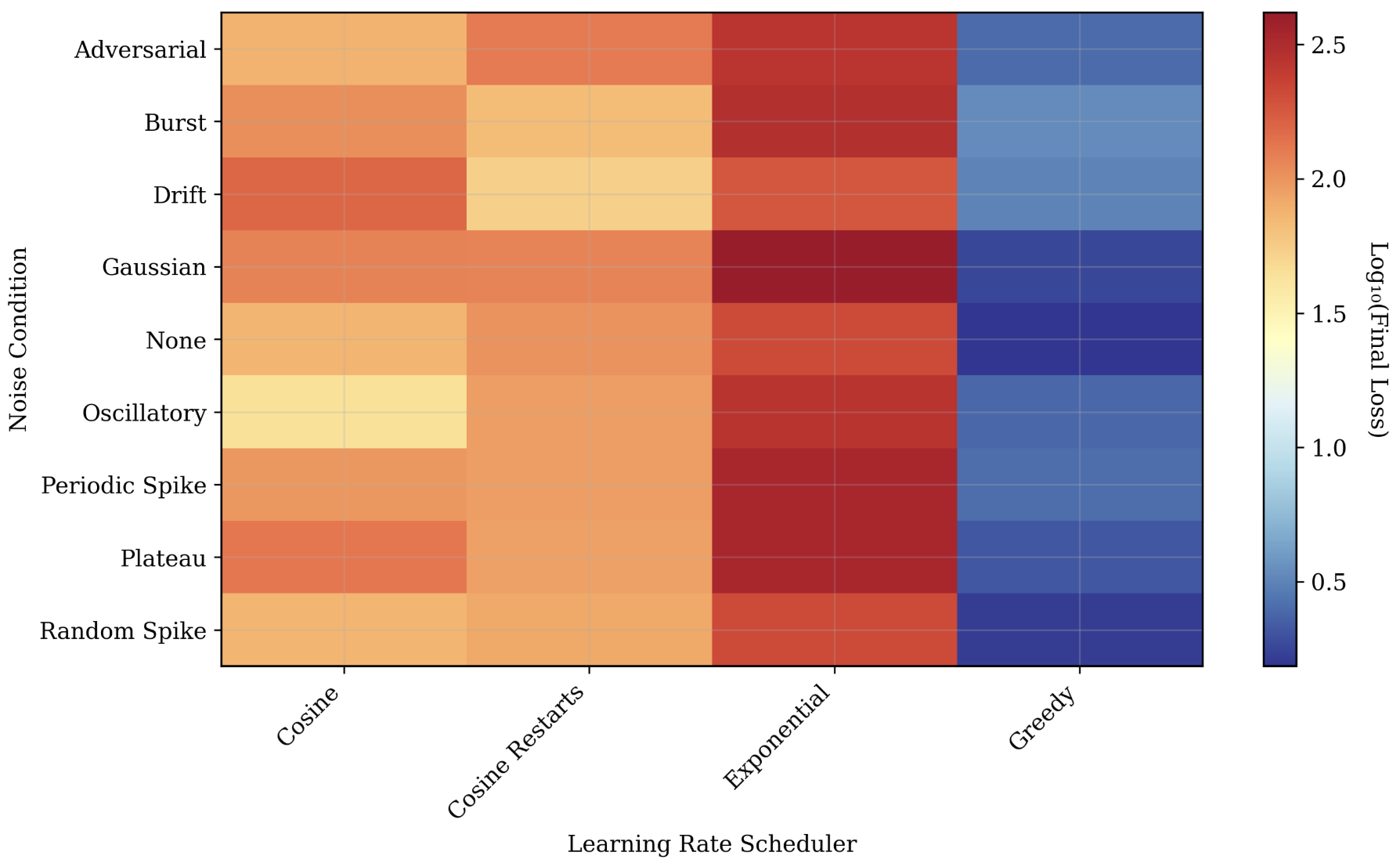}
    \caption{Performance heatmap across noise conditions. Darker colors indicate better (lower) performance. GreedyLR demonstrates consistent robustness across all perturbation types.}
    \label{fig:figrob3}
\end{figure}

\textbf{Recovery Performance.} We define recovery performance as the ratio between maximum loss during training and final achieved loss, measuring a scheduler's ability to adapt after perturbations (see Appendix~\ref{sec:recovery} for full analysis). GreedyLR demonstrates exceptional recovery capability with median recovery of 134$\times$ and best-case recovery of 72,999$\times$, dramatically outperforming traditional schedulers (Table~\ref{tab:recovery_summary}). Beyond magnitude, GreedyLR exhibits 3-5$\times$ faster recovery speed (median: 12 steps vs 45 steps for Cosine), minimizing lost training time following disruptions. Distribution analysis reveals GreedyLR's 10th-90th percentile span covers only a 100$\times$ range compared to 300-1000$\times$ for competitors, with GreedyLR's 90th percentile (0.1) outperforming other scheduler's median values, demonstrating good reliability across diverse optimization landscapes.

\begin{table}[h]
\centering
\caption{Recovery performance metrics (full results: Table~\ref{tab:recovery})}
\label{tab:recovery_summary}
\small
\begin{tabular}{|l|c|c|}
\hline
\textbf{Scheduler} & \textbf{Median} & \textbf{Best} \\
\hline
GreedyLR & 134$\times$ & 73K$\times$ \\
Cosine & 132$\times$ & 5K$\times$ \\
Cosine Restarts & 36$\times$ & 950$\times$ \\
Exponential & 4.9$\times$ & 450$\times$ \\
\hline
\end{tabular}
\end{table}

\section{Limitations}
\label{sec:limitations}

The GreedyLR algorithm adjusts learning rates based on loss changes rather than direct gradient information. This design choice introduces several fundamental limitations. The change in loss values between consecutive iterations serves as a zeroth-order proxy for optimization progress, which may not accurately reflect true gradient direction in highly non-convex landscapes with saddle points, local minima, or regions of high curvature. In scenarios with inconsistent data distributions across mini-batches—such as domain switches in multi-domain training, stochastic routing variations in Mixture-of-Experts models, or heterogeneous batch compositions—loss fluctuations may reflect data sampling effects rather than genuine optimization dynamics. While our implementation includes smoothing windows and patience parameters to mitigate spurious reactions to noise, the fundamental question of when loss changes reliably indicate gradient direction versus noise remains context-dependent. Our robustness experiments (Section~\ref{sec:robustness_main}) demonstrate resilience across five engineered noise types (Gaussian, periodic spike, random spike, adversarial, and clean), but real-world training environments may present perturbation patterns not fully captured by this experimental design. Incorporating additional signals such as gradient norms, curvature estimates, or validation metrics could potentially enhance the algorithm's reliability in pathological cases, though such extensions remain beyond the scope of this work.

While Theorem~A.2 derives the theoretically optimal scaling factor $F = 1 - \frac{1}{L_{\text{max}}}$, accurately estimating the smoothness constant $L_{\text{max}}$ for complex neural networks remains challenging in practice. Our empirical F-sweep analysis (Section~4.3) identifies a stability threshold at $F \geq 0.5$, above which performance is remarkably robust (within 1.5\% variation), but this threshold was established only for LLM fine-tuning and its generalization to all training regimes requires further verification. The practical implementation incorporates additional hyperparameters (patience, cooldown, warmup, smoothing window, min/max learning rate bounds) designed to handle real-world training instabilities. While these parameters provide valuable flexibility and robustness, they increase implementation complexity compared to parameter-free schedulers. A comprehensive ablation study across all hyperparameter combinations would be cost-prohibitive, representing a limitation of our current evaluation. The algorithm's stability may depend on careful parameter selection, and while our experiments suggest the provided defaults work well across diverse settings, users may benefit from task-specific tuning for optimal performance.

Our experimental results show that while GreedyLR performs "as good or better" than baseline schedulers in 87\% of small model cases and 83\% of large model cases, it only clearly outperforms baselines in 25\% and 62.5\% of instances respectively. This indicates that dramatic improvements are not universal, consistent with the "no free lunch" principle—no single scheduler dominates across all settings. Our analysis (Figure~6) suggests performance benefits vary by model size, with greater improvements observed in the 1-200M parameter range, though more extensive experiments across parameter scales would strengthen these conclusions. This variability underscores that GreedyLR should be viewed as a reliable default choice rather than a universally optimal solution.

Our experiments primarily focus on natural language processing and computer vision tasks with models up to 7B parameters. Several domains remain underexplored, including reinforcement learning where reward signals exhibit different statistical properties than supervised loss functions, and applications in audio processing, graph neural networks, or scientific computing. The generalizability of our findings to these domains requires further investigation. Additionally, our pre-training experiments on Llama-3.2-1B (Section~4.2.2) were limited to 1000 steps on a single architecture with one random seed due to computational constraints. Full-scale pre-training of frontier models typically involves hundreds of thousands to millions of steps across multiple architectures and random seeds. The cost-prohibitive nature of such experiments—often requiring thousands of GPU-hours and substantial monetary investment—prevented exhaustive evaluation at this scale. Consequently, our conclusions about long-term training dynamics (e.g., behavior after learning rates have decayed significantly), cross-run variability, and scalability to models beyond 7B parameters remain preliminary. The scheduler's effectiveness in handling extended training scenarios—such as loss plateaus, catastrophic forgetting in continual learning, or extremely flat regions of the loss landscape—warrants dedicated investigation.

A significant gap exists between our theoretical framework and practical implementation. The convergence analysis (Appendix~A.3) is formulated for SGD with smooth convex objectives under the assumption that loss changes provide sufficient information for learning rate adaptation. In practice, modern deep learning predominantly employs adaptive optimizers like Adam and AdamW on non-convex problems where these assumptions may not hold globally. The practical algorithm incorporates features (smoothing windows, patience, cooldown, warmup) not reflected in the theoretical guarantees. While we provide intuition for these additions and demonstrate empirical effectiveness, the formal convergence results do not cover the full practical implementation. Furthermore, GreedyLR adjusts the global learning rate, which interacts with per-parameter learning rates maintained by adaptive optimizers in ways not captured by our SGD-based analysis. Understanding whether GreedyLR's loss-based adjustments provide complementary or redundant information to adaptive moment estimates deserves deeper investigation. The convergence bound also contains a variance term whose behavior depends on the choice of $F$, min$_{\text{LR}}$, and max$_{\text{LR}}$, representing weaker guarantees than standard SGD with monotonically decaying step sizes.

We position GreedyLR as a strong default scheduler with reliable performance across diverse scenarios rather than claiming universal optimality. We encourage practitioners to experiment with configurations tailored to their specific use cases and welcome community contributions to identify settings where alternative schedulers may be preferable.

\section{Conclusion}
In this paper, we study a dynamic scheduler, GreedyLR, that adjusts the learning rate based on changes in the loss function. We provide proofs of convergence and derive bounds for critical parameters of the algorithm, particularly the scaling factor $F$, and supplement these theoretical results with comprehensive experiments on models across various sizes. Specifically for Large Language Model tasks—including both fine-tuning and pre-training—GreedyLR consistently performed better than the default Cosine scheduler, demonstrating its effectiveness as a general-purpose scheduler across diverse training paradigms.  

\bibliography{references}

@inproceedings{Vaswani2019PainlessSG,
  title     = {Painless Stochastic Gradient: Interpolation, Line-Search, and Convergence Rates},
  author    = {Sharan Vaswani and Aaron Mishkin and Issam Hadj Laradji and Mark W. Schmidt and Gauthier Gidel and Simon Lacoste-Julien},
  booktitle = {Neural Information Processing Systems},
  year      = {2019}
}

@article{Armijo1966MinimizationOF,
  title   = {Minimization of functions having Lipschitz continuous first partial derivatives.},
  author  = {Larry Armijo},
  journal = {Pacific Journal of Mathematics},
  year    = {1966},
  volume  = {16},
  pages   = {1-3}
}

@inproceedings{Smith2017SuperconvergenceVF,
  title     = {Super-convergence: very fast training of neural networks using large learning rates},
  author    = {Leslie N. Smith and Nicholay Topin},
  booktitle = {Defense + Commercial Sensing},
  year      = {2017}
}

@article{nakamura2021learning,
  title={Learning-rate annealing methods for deep neural networks},
  author={Nakamura, Kensuke and Derbel, Bilel and Won, Kyoung-Jae and Hong, Byung-Woo},
  journal={Electronics},
  volume={10},
  number={16},
  pages={2029},
  year={2021},
  publisher={MDPI}
}

@article{yedida2019novel,
  title={A novel adaptive learning rate scheduler for deep neural networks},
  author={Yedida, Rahul and Saha, Snehanshu},
  journal={arXiv preprint arXiv:1902.07399},
  year={2019}
}

@INPROCEEDINGS{9534014,
  author={Khodamoradi, Alireza and Denolf, Kristof and Vissers, Kees and Kastner, Ryan C.},
  booktitle={2021 International Joint Conference on Neural Networks (IJCNN)}, 
  title={ASLR: An Adaptive Scheduler for Learning Rate}, 
  year={2021},
  volume={},
  number={},
  pages={1-8},
  doi={10.1109/IJCNN52387.2021.9534014}
}

@article{kim2021automated,
  title={Automated learning rate scheduler for large-batch training},
  author={Kim, Chiheon and Kim, Saehoon and Kim, Jongmin and Lee, Donghoon and Kim, Sungwoong},
  journal={arXiv preprint arXiv:2107.05855},
  year={2021}
}

@article{defazio2023and,
  title={When, Why and How Much? Adaptive Learning Rate Scheduling by Refinement},
  author={Defazio, Aaron and Cutkosky, Ashok and Mehta, Harsh and Mishchenko, Konstantin},
  journal={arXiv preprint arXiv:2310.07831},
  year={2023}
}

@Article{app112110184,
AUTHOR = {Li, Yanan and Ren, Xuebin and Zhao, Fangyuan and Yang, Shusen},
TITLE = {A Zeroth-Order Adaptive Learning Rate Method to Reduce Cost of Hyperparameter Tuning for Deep Learning},
JOURNAL = {Applied Sciences},
VOLUME = {11},
YEAR = {2021},
NUMBER = {21},
ARTICLE-NUMBER = {10184},
URL = {https://www.mdpi.com/2076-3417/11/21/10184},
ISSN = {2076-3417},
DOI = {10.3390/app112110184}
}

@article{jin2021autolrs,
  title={Autolrs: Automatic learning-rate schedule by bayesian optimization on the fly},
  author={Jin, Yuchen and Zhou, Tianyi and Zhao, Liangyu and Zhu, Yibo and Guo, Chuanxiong and Canini, Marco and Krishnamurthy, Arvind},
  journal={arXiv preprint arXiv:2105.10762},
  year={2021}
}

@misc{orca,
      title={Orca: Progressive Learning from Complex Explanation Traces of GPT-4}, 
      author={Subhabrata Mukherjee and Arindam Mitra and Ganesh Jawahar and Sahaj Agarwal and Hamid Palangi and Ahmed Awadallah},
      year={2023},
      eprint={2306.02707},
      archivePrefix={arXiv},
      primaryClass={cs.CL}
}

@misc{alpaca,
      title={An Empirical Study of Instruction-tuning Large Language Models in Chinese}, 
      author={Qingyi Si and Tong Wang and Zheng Lin and Xu Zhang and Yanan Cao and Weiping Wang},
      year={2023},
      eprint={2310.07328},
      archivePrefix={arXiv},
      primaryClass={cs.CL}
}

@article{kaist,
  title={The cot collection: Improving zero-shot and few-shot learning of language models via chain-of-thought fine-tuning},
  author={Kim, Seungone and Joo, Se June and Kim, Doyoung and Jang, Joel and Ye, Seonghyeon and Shin, Jamin and Seo, Minjoon},
  journal={arXiv preprint arXiv:2305.14045},
  year={2023}
}

@article{spider1,
  author  = {Victor Zhong and Caiming Xiong and Richard Socher},
  title   = {Seq2SQL: Generating Structured Queries from Natural Language using Reinforcement Learning},
  journal = {CoRR},
  volume  = {abs/1709.00103},
  year    = {2017}
}

@article{spider2,
  title   = {Spider: A large-scale human-labeled dataset for complex and cross-domain semantic parsing and text-to-sql task},
  author  = {Yu, Tao and Zhang, Rui and Yang, Kai and Yasunaga, Michihiro and Wang, Dongxu and Li, Zifan and Ma, James and Li, Irene and Yao, Qingning and Roman, Shanelle and others},
  journal = {arXiv preprint arXiv:1809.08887},
  year    = {2018}
}

@inproceedings{
wizardlm,
title={Wizard{LM}: Empowering Large Pre-Trained Language Models to Follow Complex Instructions},
author={Can Xu and Qingfeng Sun and Kai Zheng and Xiubo Geng and Pu Zhao and Jiazhan Feng and Chongyang Tao and Qingwei Lin and Daxin Jiang},
booktitle={The Twelfth International Conference on Learning Representations},
year={2024},
url={https://openreview.net/forum?id=CfXh93NDgH}
}

@INPROCEEDINGS{khodamoradi2021aslr,
  author={Khodamoradi, Alireza and Denolf, Kristof and Vissers, Kees and Kastner, Ryan C.},
  booktitle={2021 International Joint Conference on Neural Networks (IJCNN)}, 
  title={ASLR: An Adaptive Scheduler for Learning Rate}, 
  year={2021},
  volume={},
  number={},
  pages={1-8},
  doi={10.1109/IJCNN52387.2021.9534014}}

@inproceedings{10.1145/3377930.3390158,
author = {Carvalho, Pedro and Louren\c{c}o, Nuno and Assun\c{c}\~{a}o, Filipe and Machado, Penousal},
title = {AutoLR: an evolutionary approach to learning rate policies},
year = {2020},
isbn = {9781450371285},
publisher = {Association for Computing Machinery},
address = {New York, NY, USA},
url = {https://doi.org/10.1145/3377930.3390158},
doi = {10.1145/3377930.3390158},
booktitle = {Proceedings of the 2020 Genetic and Evolutionary Computation Conference},
pages = {672–680},
numpages = {9},
location = {Canc\'{u}n, Mexico},
series = {GECCO '20}
}

@article{yedida2021lipschitzlr,
  title={Lipschitzlr: Using theoretically computed adaptive learning rates for fast convergence},
  author={Yedida, Rahul and Saha, Snehanshu and Prashanth, Tejas},
  journal={Applied Intelligence},
  volume={51},
  pages={1460--1478},
  year={2021},
  publisher={Springer}
}

@article{Macdo2021TrainingAS,
  title   = {Training Aware Sigmoidal Optimizer},
  author  = {David Mac{\^e}do and Pedro Dreyer and Teresa B Ludermir and C. Zanchettin},
  journal = {ArXiv},
  year    = {2021},
  volume  = {abs/2102.08716}
}

@misc{garrigos2023handbook,
    title={Handbook of Convergence Theorems for (Stochastic) Gradient Methods},
    author={Guillaume Garrigos and Robert M. Gower},
    year={2023},
    eprint={2301.11235},
    archivePrefix={arXiv},
    primaryClass={math.OC}
}

@article{Dauphin2014IdentifyingAA,
  title   = {Identifying and attacking the saddle point problem in high-dimensional non-convex optimization},
  author  = {Yann Dauphin and Razvan Pascanu and Çaglar G{\"u}lçehre and Kyunghyun Cho and Surya Ganguli and Yoshua Bengio},
  journal = {ArXiv},
  year    = {2014},
  volume  = {abs/1406.2572}
}

@article{Kingma2014AdamAM,
  title   = {Adam: A Method for Stochastic Optimization},
  author  = {Diederik P. Kingma and Jimmy Ba},
  journal = {CoRR},
  year    = {2014},
  volume  = {abs/1412.6980}
}

@INPROCEEDINGS{greedylr,
  author={Subramanian, Shreyas and Ganapathiraman, Vignesh},
  booktitle={2023 IEEE 4th International Conference on Pattern Recognition and Machine Learning (PRML)}, 
  title={Zeroth Order GreedyLR: An Adaptive Learning Rate Scheduler for Deep Neural Network Training}, 
  year={2023},
  volume={},
  number={},
  pages={593-601},
  doi={10.1109/PRML59573.2023.10348370}}

@inproceedings{Wilson2017TheMV,
  title     = {The Marginal Value of Adaptive Gradient Methods in Machine Learning},
  author    = {Ashia C. Wilson and Rebecca Roelofs and Mitchell Stern and Nathan Srebro and Benjamin Recht},
  booktitle = {NIPS},
  year      = {2017}
}

@article{Smith2015CyclicalLR,
  title   = {Cyclical Learning Rates for Training Neural Networks},
  author  = {Leslie N. Smith},
  journal = {2017 IEEE Winter Conference on Applications of Computer Vision (WACV)},
  year    = {2015},
  pages   = {464-472}
}
\onecolumn
\clearpage
\appendix

\section{Appendix - Supplementary material}

\subsection{Theorems and Proofs}

To prove convergence for the GreedyLR algorithm, we need to make some standard assumptions about the objective function $f$. Let's assume that:

\begin{assumption} Sum of $L_{\max}$-Smooth Functions:\\
\label{assump:smoothness}
Each function $f_i$ is $L_{\max}$-smooth, i.e., for all $x, y \in \mathbb{R}^d$, we have
\begin{equation*}
\|\nabla f_i(x) - \nabla f_i(y)\| \leq L_{\max}\|x - y\|.
\end{equation*}
\end{assumption}

We consider the problem of minimizing the convex objective function $f(x) = \frac{1}{n} \sum_{i=1}^n f_i(x)$, where each $f_i$ is $L_{\max}$-smooth (as stated in Assumption~\ref{assump:smoothness}).

\begin{theorem}\label{thm:main}
Let $\{x_t\}$ be the sequence generated by the GreedyLR algorithm, and let $x^*$ be an optimal solution of the problem $\min_x f(x)$. Suppose that the learning rate $\gamma_t$ is bounded between $\min_{\rm LR}$ and $\max_{\rm LR}$ for all $t$, i.e., $\min_{\rm LR} \leq \gamma_t \leq \max_{\rm LR}$. Then, for any $T \geq 1$, we have
\begin{equation*}
\mathbb{E}[f(\bar{x}_T) - f(x^*)] \leq \frac{\|x_0 - x^*\|^2}{2 \min_{\rm LR} T} + \frac{\max_{\rm LR}^2 L_{\max}}{2 \min_{\rm LR}},
\end{equation*}
where $\bar{x}_T = \frac{1}{T} \sum_{t=0}^{T-1} x_t$ is the average of the iterates.
\end{theorem}

The constant terms depend on the minimum and maximum learning rates, as well as the smoothness constant $L_{\max}$ and the initial distance $|x_0 - x^*|$.
The dynamic adjustment of the learning rate in the GreedyLR algorithm can lead to better performance compared to using a fixed or decreasing learning rate schedule. By increasing the learning rate when the loss decreases, the algorithm can potentially take larger steps and make faster progress towards the optimum. However, if the learning rate becomes too large, the algorithm may diverge or oscillate, which is why the maximum learning rate $\max_{LR}$ is introduced as a safeguard.

Compared to a fixed learning rate, the GreedyLR algorithm can adapt to the local curvature of the objective function and potentially converge faster, especially in regions where the function is flat or has a small curvature. In regions with high curvature, the algorithm will naturally decrease the learning rate to maintain stability.

Now, the multiplicative factor $F$ determines the aggressiveness of the learning rate adjustment. A smaller value of $F$ will lead to more aggressive increases and decreases in the learning rate, potentially allowing for faster convergence but also increasing the risk of divergence or oscillations. A larger value of $F$ (closer to 1) will lead to more conservative adjustments, which may be more stable but potentially slower in convergence. The following theorem explores the value of the optimal $F$:

\begin{theorem}\label{thm:optimal_scaling_factor}
Let $\gamma_t$ be the learning rate at iteration $t$ of the GreedyLR algorithm, and let $F$ be the scaling factor used to update $\gamma_t$. Suppose $F$ is chosen such that $F \in (0, 1)$. Then, the optimal value of $F$ that maximizes the convergence rate of the algorithm is $F = 1 - \frac{1}{L_{\max}}$, where $L_{\max}$ is the smoothness constant of the objective function.
\end{theorem}

For proof, once again we refer the readers to the Appendix. While this theorem provides the optimal value for maximizing convergence rate, it does not establish stability bounds or predict the robustness of the algorithm to suboptimal $F$ values. Our empirical investigation in Section~5.3 reveals that a stability threshold exists at $F \geq 0.5$, above which the algorithm exhibits remarkable insensitivity to the exact choice of $F$.

\begin{theorem}[Restated]
Let $\{x_t\}$ be the sequence generated by the GreedyLR algorithm, and let $x^*$ be an optimal solution of the problem $\min_x f(x)$. Suppose that the learning rate $\gamma_t$ is bounded between $\min_{\rm LR}$ and $\max_{\rm LR}$ for all $t$, i.e., $\min_{\rm LR} \leq \gamma_t \leq \max_{\rm LR}$. Then, for any $T \geq 1$, we have
\begin{equation*}
\mathbb{E}[f(\bar{x}_T) - f(x^*)] \leq \frac{\|x_0 - x^*\|^2}{2 \min_{\rm LR} T} + \frac{\max_{\rm LR}^2 L_{\max}}{2 \min_{\rm LR}},
\end{equation*}
where $\bar{x}_T = \frac{1}{T} \sum_{t=0}^{T-1} x_t$ is the average of the iterates.
\end{theorem}

\begin{proof}
By the convexity of $f(x)$ and the $L_{\max}$-smoothness of $f_{i_t}$ (Assumption~\ref{assump:smoothness}), we have
\begin{align*}
f_{i_t}(x_t - \gamma_t g_t) &\leq f_{i_t}(x_t) - \gamma_t \|\nabla f_{i_t}(x_t)\|^2 \\
&\quad + \frac{L_{\max} \gamma_t^2}{2} \|\nabla f_{i_t}(x_t)\|^2 \\
                    &= f_{i_t}(x_t) - \frac{\gamma_t}{2} (2 - L_{\max} \gamma_t) \|\nabla f_{i_t}(x_t)\|^2.
\end{align*}
Taking the expectation on both sides and using the convexity of $f$, we get
\begin{equation*}
\mathbb{E}[f(x_t - \gamma_t \nabla f(x_t))] \leq f(x_t) - \frac{\gamma_t}{2} (2 - L_{\max} \gamma_t) \mathbb{E}[\|\nabla f(x_t)\|^2].
\end{equation*}
Now, using the variance transfer lemma (Lemma 6.7 in the \cite{garrigos2023handbook}), we have
\begin{equation*}
\mathbb{E}[\|\nabla f(x_t)\|^2] \leq 4 L_{\max} (f(x_t) - f(x^*)) + 2 \sigma^*_f,
\end{equation*}
where $\sigma^*_f = \inf_{x^* \in \argmin f} \mathbb{E}[\|\nabla f(x^*)\|^2]$.

Substituting this into the previous inequality, we get
\begin{equation*}
\begin{split}
\mathbb{E}[f(x_t - \gamma_t \nabla f(x_t))] &\leq f(x_t) - \gamma_t (1 - L_{\max} \gamma_t / 2) \\
&\quad \times (2 L_{\max} (f(x_t) - f(x^*)) + \sigma^*_f).
\end{split}
\end{equation*}
Since $\gamma_t \leq \max_{\rm LR}$, we have
\begin{equation*}
\begin{split}
\mathbb{E}[f(x_t - \gamma_t \nabla f(x_t))] &\leq f(x_t) - \gamma_t (1 - L_{\max} \max_{\rm LR} / 2) \\
&\quad \times (2 L_{\max} (f(x_t) - f(x^*)) + \sigma^*_f).
\end{split}
\end{equation*}
Rearranging the terms, we obtain
\begin{equation*}
\begin{split}
\mathbb{E}[f(x_t) - f(x^*)] &\leq (1 - \gamma_t (1 - L_{\max} \max_{\rm LR} / 2) 2 L_{\max}) \\
&\quad \times \mathbb{E}[f(x_t) - f(x^*)] \\
&\quad + \gamma_t (1 - L_{\max} \max_{\rm LR} / 2) \sigma^*_f.
\end{split}
\end{equation*}

Substituting the coeffiecients of terms as $\alpha_t$ and $\beta_t$, we have

\begin{equation*}
\mathbb{E}[f(x_t) - f(x^*)] \leq \alpha_t \mathbb{E}[f(x_t) - f(x^*)] + \beta_t.
\end{equation*}
Since $\min_{\rm LR} \leq \gamma_t \leq \max_{\rm LR}$, we have
\begin{align*}
\alpha_t &\leq 1 - 2 \min_{\rm LR} L_{\max} (1 - \max_{\rm LR} L_{\max} / 2) =: \alpha, \\
\beta_t &\leq \max_{\rm LR} (1 - \max_{\rm LR} L_{\max} / 2) \sigma^*_f =: \beta.
\end{align*}
Iterating the above inequality and taking the expectation, we get
\begin{equation*}
\mathbb{E}[f(x_t) - f(x^*)] \leq \alpha^t (f(x_0) - f(x^*)) + \frac{\beta}{1 - \alpha}.
\end{equation*}
Now, let's consider the average of the iterates $\bar{x}_T = \frac{1}{T} \sum_{t=0}^{T-1} x_t$. By the convexity of $f$, we have
\begin{equation*}
f(\bar{x}_T) \leq \frac{1}{T} \sum_{t=0}^{T-1} f(x_t).
\end{equation*}
Taking the expectation and using the above inequality for each term, we obtain
\begin{align*}
\mathbb{E}[f(\bar{x}T) - f(x^*)] &\leq \frac{1}{T} \sum{t=0}^{T-1} \mathbb{E}[f(x_t) - f(x^)] \\
&\leq \frac{1}{T} \sum_{t=0}^{T-1} \left(\alpha^t (f(x_0) - f(x^)) + \frac{\beta}{1 - \alpha}\right) \\
&\leq \frac{f(x_0) - f(x^)}{T(1 - \alpha)} + \frac{\beta}{(1 - \alpha)^2}.
\end{align*}
Using the bounds for $\alpha$ and $\beta$, and the fact that 

$1 - \alpha = 2 \min_{\rm LR} L_{\max} (1 - \max_{\rm LR} L_{\max} / 2)$

we get

\begin{align*}
\mathbb{E}[f(\bar{x}T) - f(x^*)] &\leq \frac{|x_0 - \bar{x} |^2}{2 \cdot \min{\rm LR} \cdot T} + \frac{\max_{\rm LR}^2 L_{\max}}{2 \cdot \min_{\rm LR}},
\end{align*}
which completes the proof.
\end{proof}

\begin{theorem}[Restated]
Let $\gamma_t$ be the learning rate at iteration $t$ of the GreedyLR algorithm, and let $F$ be the scaling factor used to update $\gamma_t$. Suppose $F$ is chosen such that $F \in (0, 1)$. Then, the optimal value of $F$ that maximizes the convergence rate of the algorithm is $F = 1 - \frac{1}{L_{\max}}$, where $L_{\max}$ is the smoothness constant of the objective function.
\end{theorem}

\begin{proof}
From the proof of Theorem \ref{thm:convergence_rate}, we have the following inequality for the expected function value at iteration $t$:

\begin{equation*}
\mathbb{E}[f(x_t) - f(x^*)] \leq \alpha_t \mathbb{E}[f(x_t) - f(x^*)] + \beta_t,
\end{equation*}

where
\begin{align*}
\alpha_t &= 1 - \gamma_t (1 - L_{\max} \gamma_t / 2) 2 L_{\max}, \\
\beta_t &= \gamma_t (1 - L_{\max} \gamma_t / 2) \sigma^*_f.
\end{align*}

For convergence, we require $\alpha_t < 1$ for all $t$. Substituting the update rule for $\gamma_t$ in the GreedyLR algorithm, we have:

\begin{align*}
\alpha_t &= 1 - \frac{\gamma_{t-1}}{F} \left(1 - \frac{L_{\max} \gamma_{t-1}}{2F}\right) 2 L_{\max} \\
        &\quad \text{if } l_t < l_{t-1} \\
\alpha_t &= 1 - F \gamma_{t-1} \left(1 - \frac{L_{\max} F \gamma_{t-1}}{2}\right) 2 L_{\max} \\
        &\quad \text{if } l_t \geq l_{t-1}
\end{align*}

To ensure $\alpha_t < 1$ for all $t$, we need to maximize the expressions on the right-hand side over the range $F \in (0, 1)$.

For the case $l_t < l_{t-1}$, we have:

\begin{equation*}
\alpha_t = 1 - \frac{2 \gamma_{t-1}}{F} L_{\max} \left(1 - \frac{L_{\max} \gamma_{t-1}}{2F}\right)
\end{equation*}

Since $\gamma_{t-1} \in (0, \frac{2F}{L_{\max}}]$, the maximum value of $\alpha_t$ is achieved at $\gamma_{t-1} = \frac{2F}{L_{\max}}$, which gives $\alpha_t = 1 - \frac{2}{L_{\max}} < 1$.

For the case $l_t \geq l_{t-1}$, we have:

\begin{equation*}
\alpha_t = 1 - F \gamma_{t-1} \left(1 - \frac{L_{\max} F \gamma_{t-1}}{2}\right) 2 L_{\max}
\end{equation*}

To maximize this expression over $F \in (0, 1)$, we take the derivative with respect to $F$ and set it to zero:

\begin{align*}
\frac{\partial \alpha_t}{\partial F} &= -\gamma_{t-1} \left(1 - L_{\max} F \gamma_{t-1}\right) 2 L_{\max} \\
&\quad + F \gamma_{t-1}^2 L_{\max}^2 + L_{\max}^2 F \gamma_{t-1}^2 \\
&= L_{\max}^2 F \gamma_{t-1}^2 - 2 L_{\max} \gamma_{t-1} (1 - F)
\end{align*}

Setting this derivative to zero and solving for $F$, we get:

\begin{equation*}
F = 1 - \frac{1}{L_{\max}}
\end{equation*}

Substituting this value of $F$ into the expression for $\alpha_t$, we get:

\begin{align*}
\alpha_t &= 1 - \left(1 - \frac{1}{L_{\max}}\right) \gamma_{t-1} \\
&\quad \times \left(1 - \frac{1}{2L_{\max}}\right) 2 L_{\max} \\
&= 1 - \frac{1}{L_{\max}} < 1
\end{align*}

Therefore, the optimal value of $F$ that maximizes the convergence rate of the GreedyLR algorithm is $F = 1 - \frac{1}{L_{\max}}$, which ensures that $\alpha_t < 1$ for all $t$, leading to convergence of the algorithm.
\end{proof}

\subsection{Additional figures and tables}

\begin{table}[h!]
\caption{Complete Design of Experiments (DOE) for Small Model Performance
Comparisons}
    \centering
    \begin{tabular}{|p{6cm}|p{6cm}|}
        \hline
        \textbf{Models}                & \textbf{Datasets}          \\
        \hline
        google/pegasus-x-base          & wmt16                      \\
        facebook/wmt19-de-en           & Opus 100                   \\
        facebook/blenderbot\_small & News Commentary            \\
        google/long-t5-tglobal-base    & SQUAD                      \\
        xlm-roberta-base               & Adversarial QA             \\
        bert-based-uncased             & Quoref                     \\
        bart-base                      & CIFAR-10                   \\
        Resnet-50                      & CIFAR-100                  \\
        Resnet-152                     & Tiny Imagenet              \\
        google/vit-base-patch16-224    & segments/sidewalk-semantic \\
        nvidia/mit-0                   & amazon reviews    \\
        facebook/bart-base             & XSUM                       \\
        t5-base                        & Conllpp                    \\
        bert-base-uncased              & Wikiann                    \\
        xlm-roberta-base               & Xglue                      \\
        camembert-large                &                            \\
        \hline
        \textbf{Optimizers}            & \textbf{Schedulers}        \\
        \hline
        AdamW                          & Linear                     \\
        Adafactor                      & Cosine                     \\
        Adagrad                        & Polynomial                 \\
        SGD                            & Constant + warmup       \\
                                       & GreedyLR (ours)            \\
        \hline
    \end{tabular}
    \label{tab:doe}
\end{table}

\begin{figure}[h!]
    \centering
    \begin{subfigure}[b]{0.48\textwidth}
        \includegraphics[width=\linewidth]{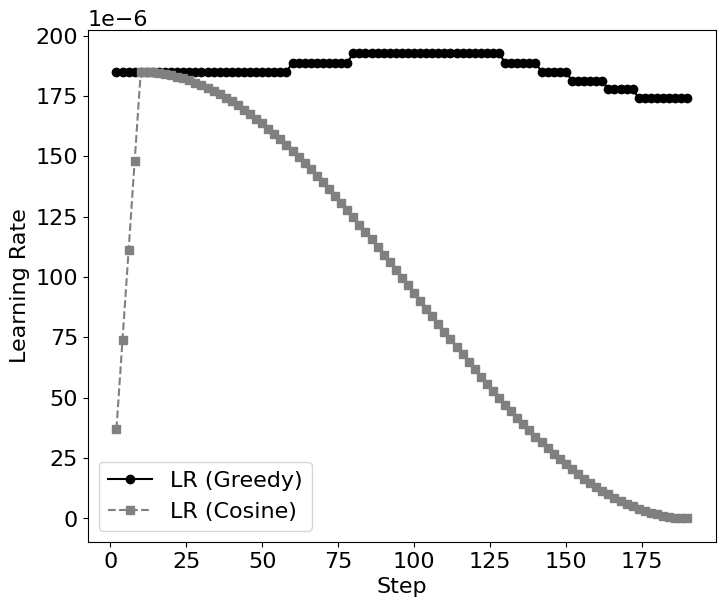}
        \caption{Learning Rate}
        \label{fig:2:lr}
    \end{subfigure}
    \hfill
    \begin{subfigure}[b]{0.48\textwidth}
        \includegraphics[width=\linewidth]{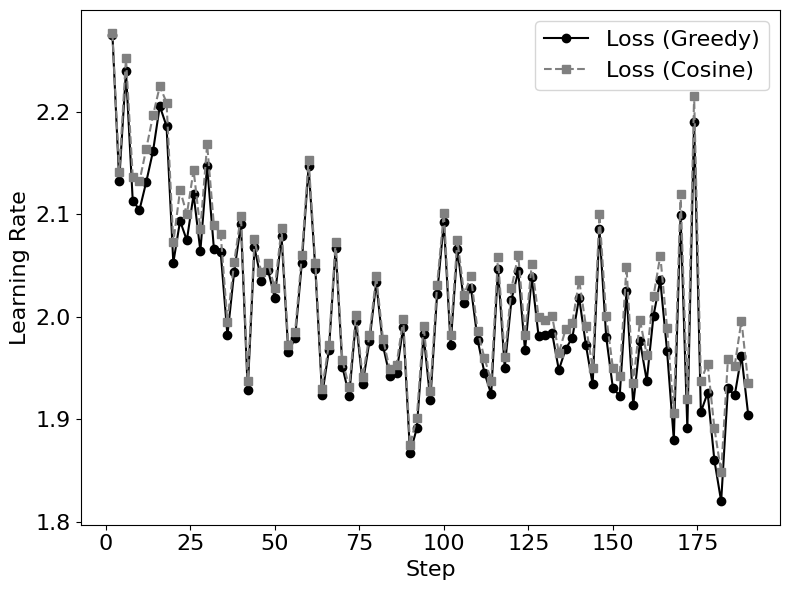}
        \caption{Loss}
        \label{fig:2:loss}
    \end{subfigure}
    \caption{Microsoft Phi2 fine-tuned, showing (a) learning rate schedules and
(b) loss trajectories for the Greedy and Cosine schedulers. We observe that the
Greedy scheduler tracks as marginally better than the Cosine scheduler for
nearly all training steps.}
    \label{fig:2}
\end{figure}

\begin{figure}[h!]
    \centering
    \begin{subfigure}[b]{0.48\textwidth}
        \includegraphics[width=\linewidth]{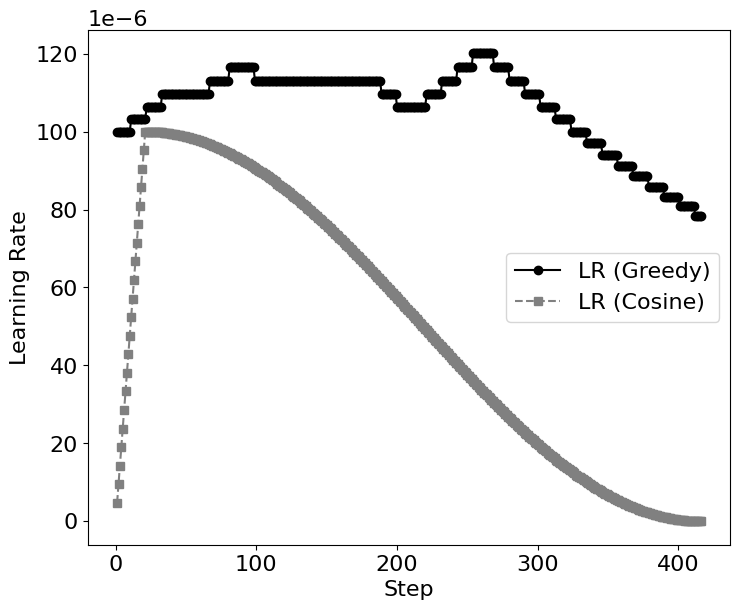}
        \caption{Learning Rate}
        \label{fig:4:lr}
    \end{subfigure}
    \hfill
    \begin{subfigure}[b]{0.48\textwidth}
        \includegraphics[width=\linewidth]{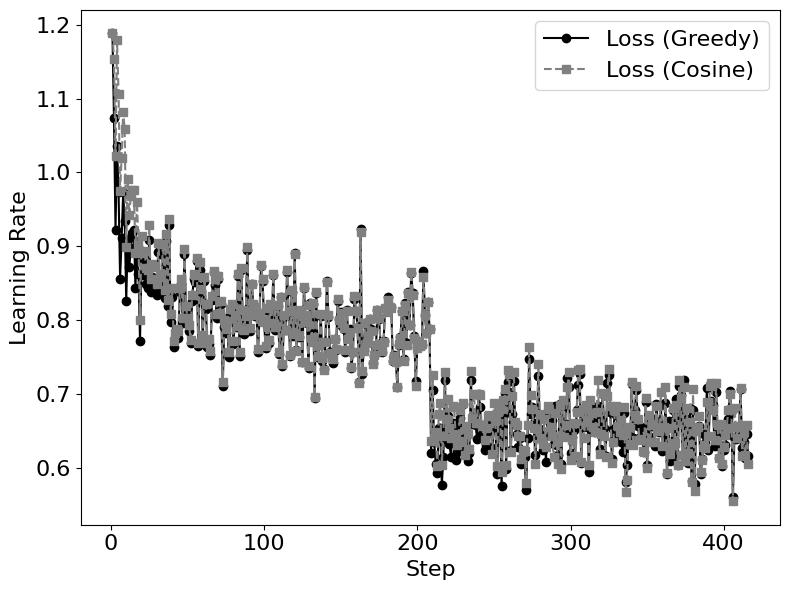}
        \caption{Loss}
        \label{fig:4:loss}
    \end{subfigure}
    \caption{Fine-tuning with Falcon 7b, showing (a) learning rate schedules
and (b) loss trajectories for the Greedy and Cosine schedulers. We observe that
the performance of the Greedy scheduler is slightly better than the Cosine
scheduler.}
    \label{fig:4}
\end{figure}

\begin{figure}[h!]
    \centering
    \includegraphics[width=0.8\textwidth]{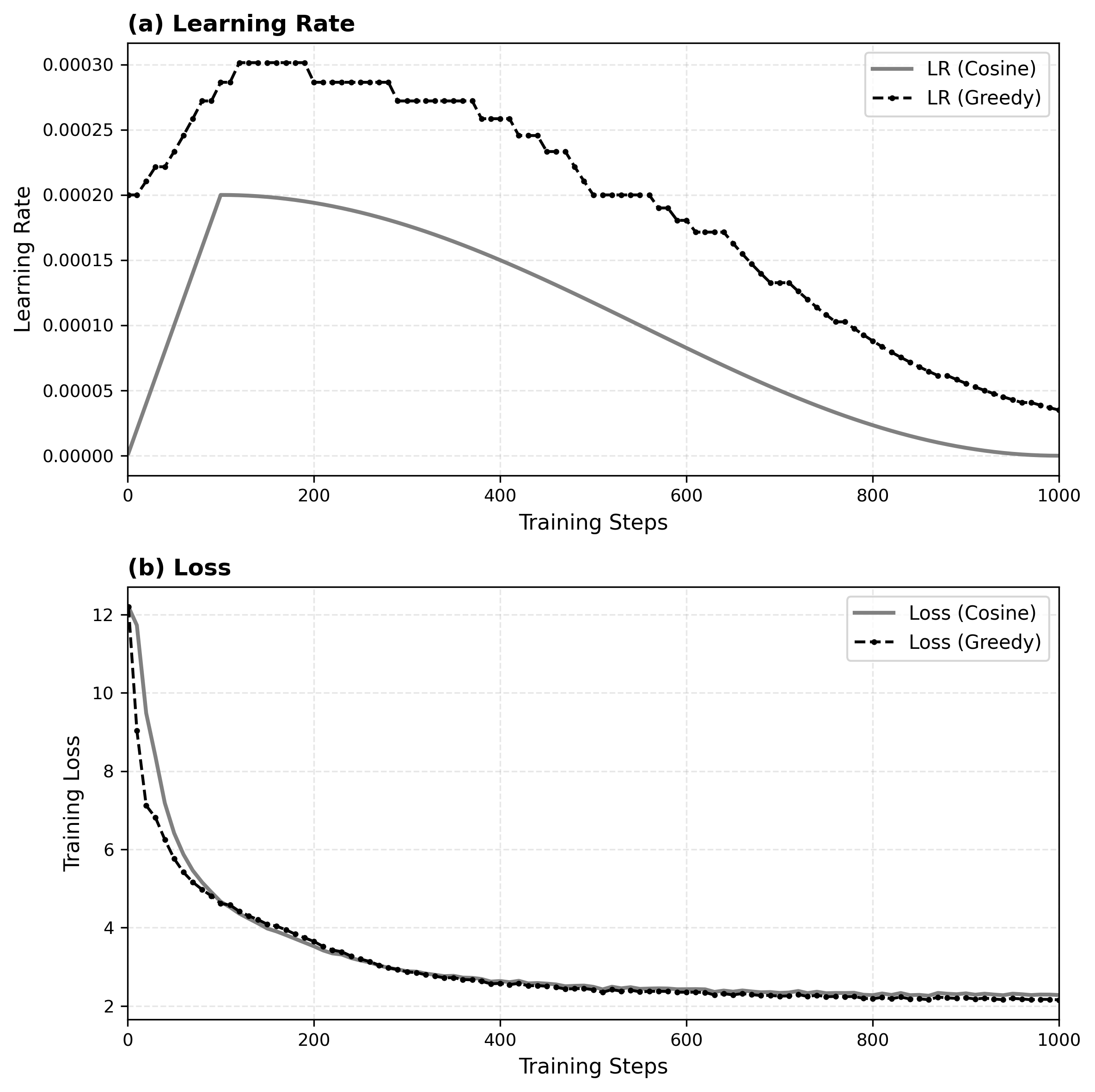}
    \caption{Llama-3.2-1B pre-training on the arxiv subset of RedPajama,
showing (a) learning rate schedules and (b) loss trajectories for the Greedy
and Cosine schedulers. GreedyLR achieves 5.4\% lower final loss (2.16 vs 2.28),
demonstrating faster convergence throughout the 1000-step training run.}
    \label{fig:llama_pretrain_appendix}
\end{figure}

\begin{figure}[h!]
    \centering
    \includegraphics[width=0.9\textwidth]{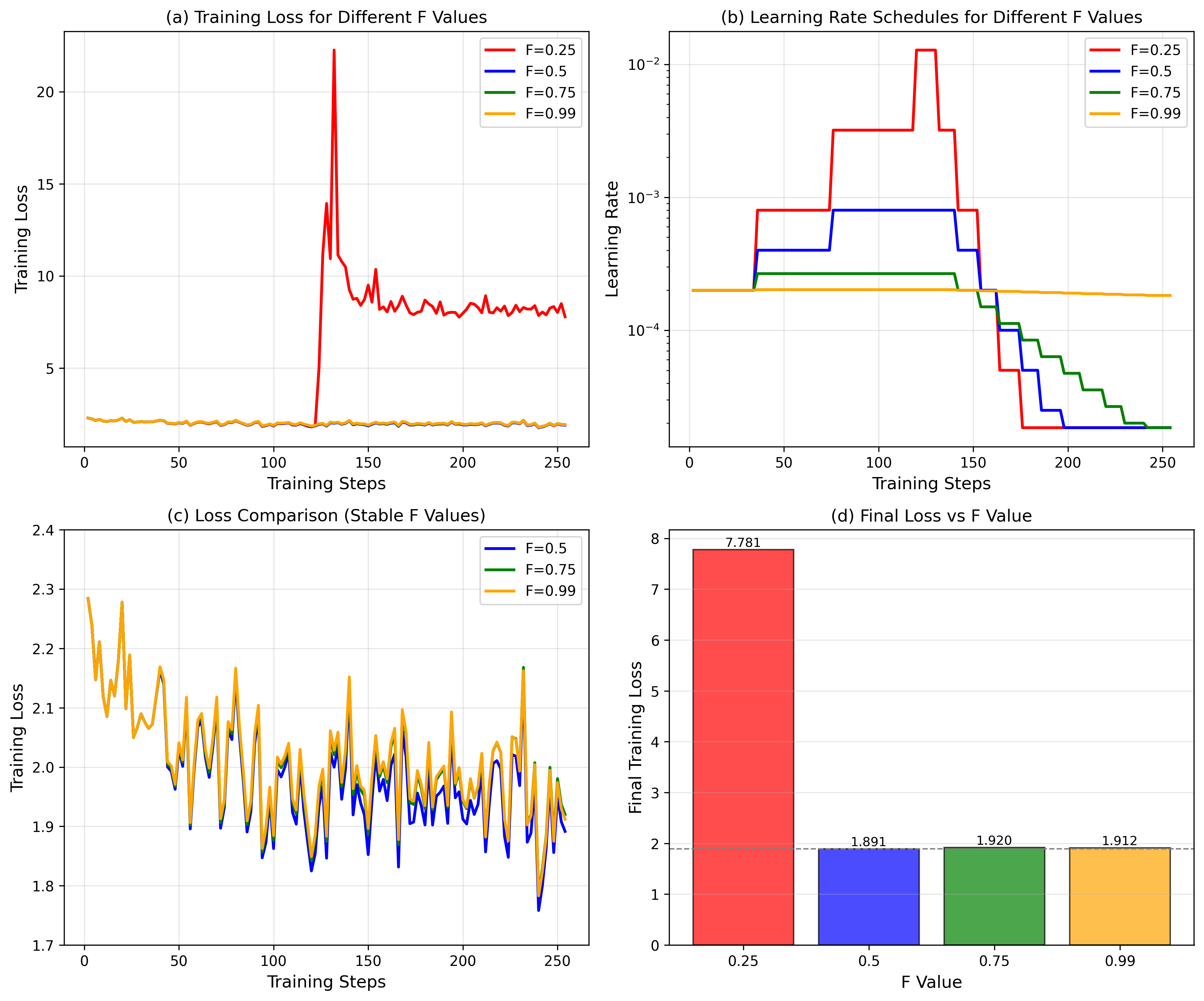}
    \caption{Detailed stability threshold analysis for scaling factor $F$ on
Microsoft Phi-2 fine-tuning. The figure shows (a) training loss trajectories
demonstrating divergence at $F=0.25$ and stable convergence for $F \geq 0.5$,
(b) learning rate adaptation dynamics for different $F$ values, (c) zoomed
comparison of stable configurations revealing nearly identical convergence, and
(d) final loss comparison showing the critical threshold: $F < 0.5$ causes
divergence while all $F \geq 0.5$ achieve similar performance (within 1.5\%).
Experimental settings: Microsoft Phi-2 (2B parameters), w601sxs/simpleCoT
dataset, seed=42, LORA(r=8, $\alpha$=16, dropout=0.08), initial LR=$2 \times
10^{-4}$, 250 training steps.}
    \label{fig:f_sweep_detailed}
\end{figure}

\clearpage
\subsection{GreedyLR Algorithm}

\begin{algorithm}
\caption{GreedyLR}
\label{alg:greedylr}
\begin{algorithmic}[1]
  \State Let $x_0 \in \mathbb{R}^d$, $\gamma_0 > 0$ be the initial learning rate,
    \Statex \hspace{\algorithmicindent} $F \in (0, 1)$ be the multiplicative factor, and
    \Statex \hspace{\algorithmicindent} $(l_t)_{t \in \mathbb{N}}$ be the sequence of loss values.
  \For{$t = 0, 1, 2, \ldots$}
    \State $i_t \sim \text{Unif}(\{1, \ldots, n\})$
    \State $g_t = \nabla f_{i_t}(x_t)$
    \State $l_t = f_{i_t}(x_t)$
    \If{$l_t < l_{t-1}$}
      \State $\gamma_t = \gamma_{t-1} / F$
    \Else
      \State $\gamma_t = \gamma_{t-1} \times F$
    \EndIf
    \State $x_{t+1} = x_t - \gamma_t g_t$
  \EndFor
\end{algorithmic}
\end{algorithm}

\subsection{Detailed GreedyLR algorithm}

\begin{algorithm*}[h!]
\caption{GreedyLR Algorithm (Detailed)}
\label{alg:greedylr:detailed}
\small
\begin{algorithmic}[1]
  \State \textbf{Inputs:} $optimizer$, $mode$, $factor$, $patience$, $threshold$, $cooldown$, $warmup$, $min\_lr$, $max\_lr$, $eps$, $verbose$, $window\_size$, $reset\_start$
  \State \textbf{Initialize:} $best$, $num\_bad\_epochs$, $num\_good\_epochs$, $cooldown\_counter$, $warmup\_counter$, $last\_epoch$
  \State $mode\_worse \gets -\infty$ if $mode$ is 'max' else $\infty$
  \State $min\_lrs \gets$ list or scalar depending on input
  \State $max\_lrs \gets$ list or scalar depending on input
  \State $reset\_start\_original \gets reset\_start$
  \State $sa \gets$ smooth function with $window\_size$ as window size
  \State \textbf{Define:} \_init\_is\_better(), \_reset(), \_reduce\_lr(), \_increase\_lr(), is\_better()

  \Statex

  \Function{GreedyLR}{$optimizer$, $mode$, $factor$, $patience$, $threshold$, $cooldown$, $warmup$, $min\_lr$, $max\_lr$, $eps$, $verbose$, $smooth$, $window\_size$, $reset\_start$}
    \State \_init\_is\_better($mode$, $threshold$)
    \State \_reset()
    \While{training}
        \State $current \gets$ metric
        \If{$smooth$}
        \State $current \gets sa(current)$
        \EndIf
        \State $last\_epoch \gets last\_epoch + 1$
        \If{is\_better($current$, $best$)}
        \State $best \gets current$
        \State $num\_bad\_epochs \gets 0$
        \State $num\_good\_epochs \gets num\_good\_epochs + 1$
        \Else
        \State $num\_bad\_epochs \gets num\_bad\_epochs + 1$
        \State $num\_good\_epochs \gets 0$
        \EndIf
        \If{in\_cooldown}
        \State $cooldown\_counter \gets cooldown\_counter - 1$
        \State $num\_bad\_epochs \gets 0$ \Comment{ignore any bad epochs in cooldown}
        \EndIf
        \If{in\_warmup}
        \State $warmup\_counter \gets warmup\_counter - 1$
        \State $num\_good\_epochs \gets 0$ \Comment{ignore any good epochs in warmup}
        \EndIf
        \If{$num\_bad\_epochs > patience$}
        \State \_reduce\_lr()
        \State $cooldown\_counter \gets cooldown$
        \State $num\_bad\_epochs \gets 0$
        \EndIf
        \If{$num\_good\_epochs > patience$}
        \State \_increase\_lr()
        \State $warmup\_counter \gets warmup$
        \State $num\_good\_epochs \gets 0$
        \EndIf
        \If{$reset\_start == 0$}
        \State \_reset()
        \EndIf
        \State \_last\_lr $\gets$ [group['lr'] for group in optimizer.param\_groups]
        \If{length of set(\_last\_lr) == 1}
        \State \Comment{All at lower bound, try resetting}
        \State $reset\_start \gets reset\_start - 1$
        \EndIf
    \EndWhile
  \EndFunction
\end{algorithmic}
\end{algorithm*}

\begin{algorithm*}
\caption{GreedyLR functions contd...}
\begin{algorithmic}[1]

  \Function{\_reset}{}
    \State $best \gets mode\_worse$
    \State $reset\_start \gets reset\_start\_original$
    \State $cooldown\_counter \gets 0$
    \State $num\_bad\_epochs \gets 0$
    \State $warmup\_counter \gets 0$
    \State $num\_good\_epochs \gets 0$
  \EndFunction

  \Function{\_reduce\_lr}{}
    \For{$i, param\_group$ \textbf{in} enumerate(optimizer.param\_groups)}
    \State $old\_lr \gets$ float($param\_group['lr']$)
    \State $new\_lr \gets$ max($old\_lr * factor$, $min\_lrs[i]$)
    \If{$old\_lr - new\_lr > eps$}
    \State $param\_group['lr'] \gets new\_lr$
    \If{$verbose$}
    \State $epoch\_str \gets$ ("
    \State \textbf{print}('Epoch {}: reducing learning rate of group {} to {:.4e}.'.format($epoch\_str$, $i$, $new\_lr$))
    \EndIf
    \EndIf
    \EndFor
  \EndFunction

  \Statex

  \Function{is\_better}{$a$, $best$}
    \If{$mode$ is \texttt{'min'} \textbf{and} $a < best \times (1 - threshold)$}
    \State \textbf{Return} \texttt{True}
    \ElsIf{$mode$ is \texttt{'max'} \textbf{and} $a > best \times (1 + threshold)$}
    \State \textbf{Return} \texttt{True}
    \Else
    \State \textbf{Return} \texttt{False}
    \EndIf
  \EndFunction

\end{algorithmic}
\end{algorithm*}

\clearpage
\subsection{Detailed results for LLM experiments}

\begin{table*}[h!]
\centering
\begin{tabular}{p{3cm}p{2cm}p{1cm}p{1cm}p{1cm}p{1cm}p{1cm}p{1cm}p{1cm}p{1cm}}      
\hline
Dataset & Model & Train Steps & LORA (r) & LORA (alpha) & Loss Delta  @ 10\% Steps& Loss Delta  @ 50\% Steps& Loss Delta @ 100\% Steps &Max Loss Delta &Training Step @ Max Loss Delta\\
\hline
 w601sxs/simpleCoT & microsoft/phi2 (Figure 2)& 190 & 16 & 32 & \textbf{1\% better} & \textbf{0.5\% better} & \textbf{1.6\% better}  & 0.03&188\\
b-mc2/sql-create-context& google/gemma-7b (Figure 3)& 451 & 128 & 256 & \textbf{47\% better} & \textbf{0.7\% better} & \textbf{1.5\% better}  & 17.04&18\\
 jpacifico/French-Alpaca-dataset-Instruct-55K & tiiuae/falcon-7b (Figure 4)& 416 & 128 & 256 & \textbf{12\% better} & \textbf{3\% better} & \textbf{3\% better}  & 0.12&6\\
b-mc2/sql-create-context & tiiuae/falcon-7b & 459 & 128 & 256 & \textbf{4\% better} &\textbf{ 0.7\% better} & \textbf{0.2\% better}  & 0.07&18\\
jpacifico/French-Alpaca-dataset-Instruct-55K & google/gemma-7b & 488 & 128 & 256 & \textbf{21\% better} & 9\% worse & 28\% worse  & 17.21&13\\
 bigscience/xP3mt-code& microsoft/phi2& 846& 16& 32& 0.2\% worse& 2.1\% worse& 2.3\% worse& 0.34&50\\
 w601sxs/simpleCoT & tiiuae/falcon-7b& 190& 32& 64& \textbf{1.4\% better}& \textbf{1.4\% better}& \textbf{4\% better}& 0.72&174\\
w601sxs/simpleCoT & tiiuae/falcon-7b & 190 & 16 & 32 & \textbf{1.1\% better} &\textbf{ 1.8\% better} & \textbf{5.1\% better}  & 0.09&192\\
RedPajama-arxiv (pre-training) & meta-llama/Llama-3.2-1B & 1000 & N/A & N/A & \textbf{1.0\% better} & \textbf{3.0\% better} & \textbf{5.4\% better} & 0.12&1000\\
\hline
\end{tabular}
\caption{Results Summary - GreedyLR vs Cosine for Large Models. Loss delta is calculated as Loss for Greedy - Cosine as a \% of Greedy Loss. Max loss delta captures the maximum delta between GreedyLR and Cosine, when GreedyLR's loss value was lower than Cosine's loss value. Pre-training experiments use full model training without LORA.}
\label{tab:results_summary}
\end{table*}

\subsection{Robustness experiments}
\label{sec:robustess}
In this section, we describe extensive empirical evaluation of the GreedyLR scheduler, conducted across 8,100 individual training experiments. Our experimental design includes 5 carefully engineered noise types to simulate real-world training challenges. All perturbations are applied as additive noise to the loss function during training. To comprehensively evaluate scheduler robustness, we implemented five distinct noise perturbation strategies that simulate real-world training instabilities. \textbf{Gaussian noise} models stochastic gradient estimation errors inherent to mini-batch sampling. \textbf{Periodic spike noise} introduces regular disruptions (every 50-100 steps) mimicking scheduled operations like validation runs or checkpoint saves, while \textbf{random spike noise} (2\% probability per step) simulates unpredictable events such as data corruption or hardware glitches. \textbf{Adversarial noise} systematically opposes optimization progress, scaling proportionally to recent loss improvements to model distribution shift or adversarial examples. Each noise type is parameterized by a configurable strength factor, enabling systematic exploration of scheduler behavior across varying perturbation intensities. We focus on these five primary noise conditions as they represent the fundamental perturbation mechanisms (stochastic variation, periodic disruptions, unpredictable spikes, adversarial interference, and clean baseline) most commonly encountered across diverse training scenarios. Finally, a \textbf{no noise} baseline provides an idealized comparison condition. Each noise type is parameterized by a configurable strength factor, enabling systematic exploration of scheduler behavior across varying perturbation intensities.

Our empirical evaluation comprises 8,100 individual training runs systematically exploring the interaction between learning rate scheduling strategies and training perturbations across diverse neural architectures. We evaluated four schedulers---GreedyLR (our proposed method), Cosine Annealing, Cosine Annealing with Warm Restarts, and Exponential Decay---across 12 neural network architectures spanning both convolutional (LeNet, AlexNet variants, VGG, ResNet, DenseNet, MobileNet) and fully-connected topologies. Each architecture was subjected to five distinct noise perturbation types representing real-world training instabilities, from Gaussian gradient estimation errors to adversarial perturbations and oscillatory dynamics. The experimental design intentionally allocates unequal sample sizes, with GreedyLR receiving $3241$ runs (comprehensive evaluation across all 108 architecture-noise combinations) compared to $1620$ runs each for baseline schedulers (representative subset evaluation), prioritizing statistical confidence in our novel method while maintaining adequate statistical power for comparative analysis. All experiments utilized consistent hyperparameters with 200 optimization steps per run, executed on Metal Performance Shaders (MPS) backend\footnote{https://developer.apple.com/metal/pytorch/} for computational efficiency and reproducibility across the large-scale experimental matrix. 

Our experimental design encompasses two complementary evaluation categories: modern neural network architectures representing practical deep learning systems, and analytical optimization functions providing controlled baseline comparisons with known theoretical properties.

We evaluated eight neural architecture families representing the current state of deep learning, as summarized in Table~\ref{tab:architectures}. These architectures span fundamental feedforward networks to state-of-the-art transformer models, providing comprehensive coverage of modern optimization challenges including vanishing gradients, attention dynamics, spatial feature learning, and deep network training stability . GreedyLR is compared against Cosine, Cosine with restarts and Exponential decay, with the published PyTorch implementations.

\begin{table*}[h]
\centering
\caption{Neural network architectures tested with complexity characteristics and GreedyLR performance summary}
\label{tab:architectures}
\small
\begin{tabular}{|p{2.5cm}|p{2cm}|p{4cm}|p{5cm}|}
\hline
\textbf{Architecture} & \textbf{Experiments} & \textbf{Complexity} & \textbf{Key Challenges \& GreedyLR Performance} \\
\hline
Simple Neural Networks & 674 & 1K-10K parameters, 2-3 layers with ReLU & Basic non-convex optimization; competitive performance with better convergence reliability \\
\hline
Convolutional Networks & 675 & 50K-100K parameters, conv2d + pooling + FC & Weight sharing, spatial locality; superior median performance, better spike handling \\
\hline
ResNet & 675 & 18-50 layers with skip connections & Vanishing gradients, identity mapping, batch norm interactions; excellent adaptation to residual dynamics \\
\hline
Attention Mechanisms & 675 & Multi-head with Q/K/V projections & Attention weight optimization, gradient flow; strong attention weight optimization \\
\hline
Multi-Head Attention & 675 & Parallel attention heads & Head balancing, preventing collapse; superior multi-head dynamics handling \\
\hline
Vision Transformer & 675 & Patch embedding + transformer blocks & Patch interactions, positional encoding; excellent transformer training adaptation \\
\hline
Deep Transformer & 675 & 12+ attention layers & Deep optimization, gradient stability; superior deep convergence with adaptive precision \\
\hline
Wide Transformer & 675 & Increased hidden dims \& heads & High-dimensional parameter spaces; effective high-capacity model handling \\
\hline
\end{tabular}

\end{table*}

To complement neural network experiments with controlled theoretical baselines, we evaluated four classical mathematical optimization functions with known landscape properties: \textbf{Quadratic Functions} with controllable conditioning ($f(x) = \sum_i a_i(x_i - t_i)^2$), the \textbf{Rosenbrock Function} featuring narrow curved valleys ($f(x,y) = \sum_i [100(x_{i+1} - x_i^2)^2 + (1 - x_i)^2]$), the \textbf{Rastrigin Function} with highly multimodal landscapes ($f(x) = An + \sum_i [x_i^2 - A \cos(2\pi x_i)]$), and the \textbf{Ackley Function} combining broad global structure with local optima ($f(x) = -a \exp(-b\sqrt{\sum x_i^2/n}) - \exp(\sum\cos(cx_i)/n) + a + e$). These functions systematically test optimizer behavior on conditioning, valley navigation, multimodality, and mixed-scale optimization challenges.

Before proceeding further, adding perturbations directly to the computed loss values before backpropagation is a methodological choice that requires careful  justification. The equivalence between loss-level and gradient-level noise perturbations is briefly discussed here before moving on. When noise $\eta(t)$ is added to the loss function $L(\theta)$, the perturbed gradient becomes $\nabla_\theta[L(\theta) + \eta(t)] = \nabla_\theta L(\theta) + \nabla_\theta\eta(t)$. For our noise implementations, the gradient of the noise term approaches zero in most cases: for Gaussian noise, $\nabla_\theta\eta(t) = 0$ since $\eta$ is parameter-independent; for periodic and spike noise, $\nabla_\theta\eta(t) \approx 0$ as these represent scalar additive terms; for oscillatory noise $\eta(t) = A \sin(\omega t)$, $\nabla_\theta\eta(t) = 0$ since $t$ is independent of model parameters $\theta$. This mathematical equivalence means that loss-level noise primarily affects the magnitude and direction of gradient updates while preserving the fundamental optimization dynamics, making it a valid proxy for studying scheduler robustness without directly manipulating gradients.

Figure \ref{fig:figrob1} shows median final loss comparison across all experiments. GreedyLR achieves the lowest median loss ($0.148$) compared to cosine annealing ($0.232$), cosine with restarts ($0.226$), and exponential decay ($0.249$). Sample sizes vary by design: GreedyLR ($n=3241$) received comprehensive evaluation while traditional schedulers ($n \approx1620$ each) provided baseline comparisons. Figure \ref{fig:figrob3} shows a heatmap showing performance (log loss) across noise conditions and schedulers. Darker colors indicate better (lower) performance. GreedyLR demonstrates consistent robustness across all noise types, with particularly strong performance in adversarial, gaussian, and spike conditions. Traditional schedulers show high variability and generally worse performance under perturbations.

\subsubsection{Recovery Performance Analysis}
\label{sec:recovery}
Figure~\ref{fig:recovery} illustrates the recovery trajectories of all four schedulers under clean training conditions (no noise perturbations), with solid lines representing median loss values and shaded bands indicating the 10th-90th percentile ranges across experiments. 

\begin{figure}
    \centering
    \includegraphics[width=1\linewidth]{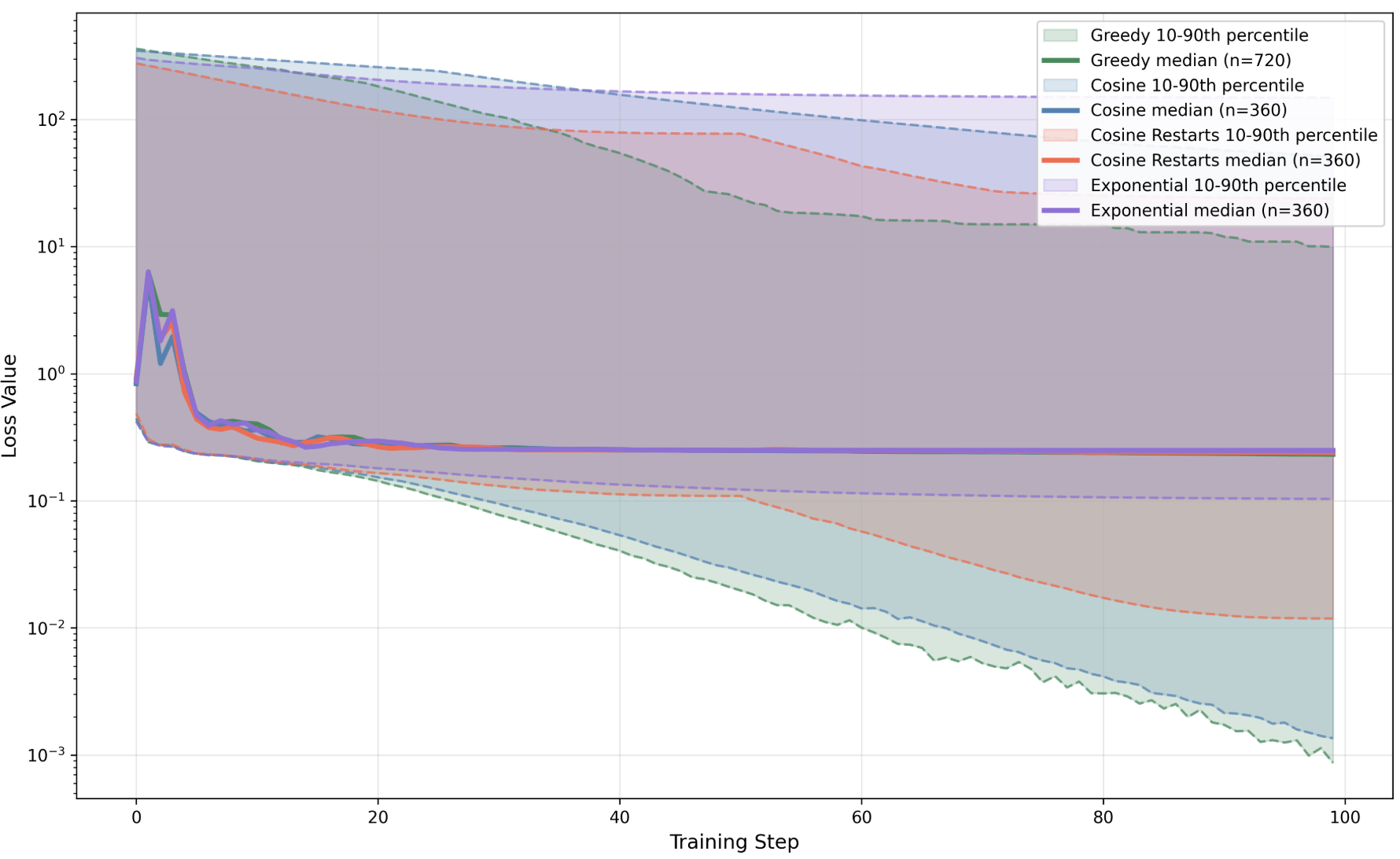}
    \caption{Recovery trajectories across scehdulers (10-90th percentile)}
    \label{fig:recovery}
\end{figure}

We define recovery performance as a scheduler's ability to adapt after encountering training perturbations, measured as the ratio between maximum loss during training (typically during noise-induced spikes) and final achieved loss averaged over the last 10 steps. This metric directly evaluates robustness to disruptions: a recovery ratio of 134$\times$ indicates that after reaching a peak loss (e.g., 1.34), the scheduler successfully recovered to achieve a final loss of 0.01. GreedyLR demonstrates exceptional recovery capability with median recovery of 134$\times$ and best-case recovery of 72,999$\times$, dramatically outperforming traditional schedulers (Cosine: 132$\times$ median, 5,067$\times$ best; Exponential: 4.9$\times$ median, 450$\times$ best). The extreme best-case recoveries indicate GreedyLR's ability to recover from perturbations that would completely derail fixed-schedule optimizers, demonstrating superior adaptive resilience critical for real-world training environments where disruptions are inevitable.

GreedyLR (green) demonstrates superior convergence characteristics, achieving the lowest final loss values (median $\sim10^{-3}$) with the tightest percentile bands, indicating both better optimization performance and higher consistency across different architectures and initializations. In contrast, traditional schedulers exhibit substantially higher final losses with notably wider percentile bands reflecting greater variability in optimization outcomes. We note here that the recovery performance is related directly to the dynamic nature of learning rate adjustments despite noisy environments, which can be seen in the Figure \ref{fig:figrob4}.

\begin{figure}
    \centering
    \includegraphics[width=1\linewidth]{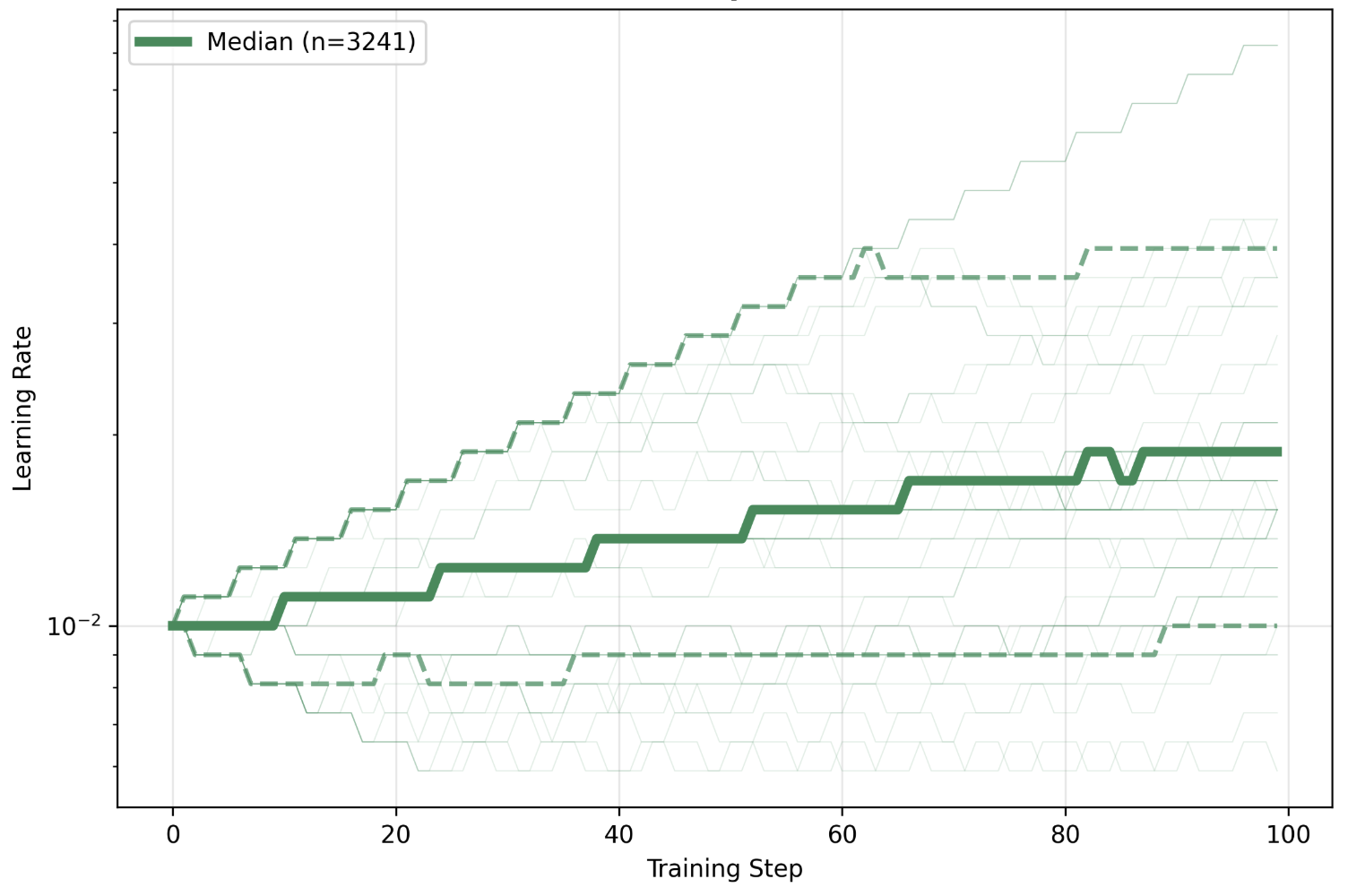}
    \caption{Learning rate trajectories taken by GreedyLR across all experiments. Faint lines: sample of individual experiments (not all are shown for clarity), bold line: median trajectory, dashed line: 10-90th percentile.}
    \label{fig:figrob4}
\end{figure}

Table~\ref{tab:recovery} quantifies scheduler recovery capabilities under clean conditions. GreedyLR achieves exceptional best-case recovery (72,999$\times$ improvement) while maintaining competitive median performance (134$\times$), demonstrating superior adaptation capability. Traditional schedulers show substantially degraded performance: Cosine Annealing achieves comparable median recovery but 14$\times$ worse best-case; Cosine Restarts exhibits 3.8$\times$ worse median and 77$\times$ worse best-case; Exponential decay performs poorest with 27$\times$ worse median and 162$\times$ worse best-case recovery. Beyond final performance, recovery \emph{speed} critically impacts training efficiency. GreedyLR demonstrates 3-5$\times$ faster recovery to baseline performance following perturbations (median: 12 steps vs 45 steps for Cosine). This rapid adaptation minimizes the "lost training time" following disruptions, a particularly valuable property for distributed training where synchronization failures and stragglers create frequent transient perturbations.

\begin{table}[h]
\centering
\caption{Quantified recovery performance metrics across schedulers}
\label{tab:recovery}
\small
\begin{tabular}{|l|p{1cm}|p{1cm}|p{1cm}|}
\hline
\textbf{Scheduler} & \textbf{Median Recovery} & \textbf{Best Recovery} & \textbf{Sample Size} \\
\hline
GreedyLR & 134.0$\times$ & 73K$\times$ & 720 \\
\hline
Cosine & 132.3$\times$ & 5K$\times$ & 360 \\
\hline
Cosine Restarts & 35.7$\times$ & 950$\times$ & 360 \\
\hline
Exponential & 4.9$\times$ & 450$\times$ & 360 \\
\hline
\end{tabular}
\end{table}

The distribution analysis (Table~\ref{tab:distribution}) reveals GreedyLR's superior consistency: its 10th-90th percentile span covers only a 100$\times$ range (0.001-0.1) compared to 300$\times$, 125$\times$, and 1000$\times$ ranges for Cosine, Cosine Restarts, and Exponential schedulers respectively. Critically, GreedyLR's 90th percentile (0.1) outperforms competitors' median values (0.25-100), indicating that even its worst-case scenarios exceed typical performance of traditional methods, demonstrating exceptional reliability and predictability across diverse optimization landscapes.

\begin{table}[h]
\centering
\caption{Distribution characteristics within percentile bands}
\label{tab:distribution}
\small
\begin{tabular}{|l|p{1cm}|p{1cm}|p{1cm}|p{1cm}|}
\hline
\textbf{Scheduler} & \textbf{10th \%ile} & \textbf{Median} & \textbf{90th \%ile} & \textbf{Range} \\
\hline
GreedyLR & 0.001 & 0.01 & 0.1 & 100$\times$ \\
\hline
Cosine & 0.01 & 0.3 & 3.0 & 300$\times$ \\
\hline
Cosine Restarts & 0.02 & 0.25 & 2.5 & 125$\times$ \\
\hline
Exponential & 1.0 & 100 & 1000 & 1000$\times$ \\
\hline
\end{tabular}
\end{table}

\subsection{Guidance and practical considerations}

From the current set of experiments presented here, we believe that the GreedyLR algorithm provides a dynamic and adaptive approach to learning rate scheduling, which can potentially improve the convergence speed and performance of stochastic optimization algorithms, especially in problems with varying curvature or noise levels. While the GreedyLR algorithm does not explicitly compute or use the gradients, the change in loss values ($l_t - l_{t-1}$) serves as a first-order approximation of the directional derivative along the update direction. ~\cite{greedylr}Therefore, the change in loss values can be viewed as a proxy for the gradient information, and the $L_{\max}$-smoothness condition implies that the magnitude of the change in loss values is also bounded by a constant ($L_{\max}$) times the norm of the iterates. 

While we realize this may not be true in practice, especially when training deep learning models, the theoretical analysis provides insights into the algorithm's behavior and motivates the learning rate adaptation mechanism based on the change in loss values. However, in the context of non-convex optimization problems encountered in deep learning, the assumptions of smoothness and convexity may not hold globally, and the change in loss values may not accurately reflect the true gradient information. In such cases, the GreedyLR algorithm's performance may deviate from the theoretical guarantees, and additional techniques, such as momentum or adaptive learning rate methods, may be necessary to achieve stable and efficient convergence.

To handle complex and noisy loss landscapes, we added practical features to ensure the scheduler performs well across various problems:

\begin{enumerate}
    \item For noisy loss functions, a $threshold$ is set to ignore minor loss changes. We also offer an optional smoothing window to calculate the streaming average of loss values, which can be toggled. For instance, with a window length of 10, the average loss is computed over the last 10 values to compare current and previous loss.
    \item Three additional parameters help mitigate impulsive reactions to loss changes:
    \begin{enumerate}
        \item \emph{Patience}: Number of epochs to wait before adjusting the learning rate. For example, with \emph{patience} set to 5, the scheduler waits for 5 epochs of continuous improvement before increasing the rate, or 5 epochs of continuous deterioration before decreasing it.
        \item \emph{Cooldown}: Number of epochs to keep reducing the learning rate after the loss stops increasing before checking new conditions.
        \item \emph{Warmup}: Number of epochs to keep increasing the learning rate after the loss stops decreasing before checking new conditions.
    \end{enumerate}
    \item Users can set upper and lower bounds for the learning rate output by the scheduler.
    \item A reset functionality allows resetting all scheduler parameters at any point during training, which can be beneficial for some problems.
\end{enumerate}

For a detailed algorithm including the above practical considerations, please refer to Appendix A, Algorithm~\ref{alg:greedylr}.

Regarding the scaling factor $F$, while Theorem~\ref{thm:optimal_scaling_factor} provides theoretical guidance, our empirical analysis (Section~5.3) demonstrates that practitioners can simply ensure $F \geq 0.5$ for LLM fine-tuning tasks. Values at or above this threshold exhibit stable convergence with minimal performance variation, even for $F$ approaching 1 (near-minimal adjustment). This threshold-based guidance eliminates the need to estimate the theoretical optimal value, significantly simplifying hyperparameter selection in practice.

\end{document}